\documentclass[11pt]{article}

\usepackage[margin=1in]{geometry}
\usepackage{amsmath,amssymb,amsthm,mathtools}
\usepackage{algorithm}
\usepackage[noend]{algpseudocode}
\usepackage{float}

\algrenewcommand\algorithmicthen{}
\usepackage[numbers,sort&compress]{natbib}
\usepackage[hidelinks]{hyperref}
\usepackage{microtype}
\usepackage{xcolor}
\usepackage{aliascnt}
\usepackage{cleveref}

\newtheorem{theorem}{Theorem}[section]
\newaliascnt{lemma}{theorem}
\newtheorem{lemma}[lemma]{Lemma}
\aliascntresetthe{lemma}
\newaliascnt{corollary}{theorem}
\newtheorem{corollary}[corollary]{Corollary}
\aliascntresetthe{corollary}
\newaliascnt{proposition}{theorem}
\newtheorem{proposition}[proposition]{Proposition}
\aliascntresetthe{proposition}
\newaliascnt{assumption}{theorem}
\newtheorem{assumption}[assumption]{Assumption}
\aliascntresetthe{assumption}
\crefname{assumption}{assumption}{assumptions}
\Crefname{assumption}{Assumption}{Assumptions}
\theoremstyle{definition}
\newaliascnt{definition}{theorem}
\newtheorem{definition}[definition]{Definition}
\aliascntresetthe{definition}
\newaliascnt{remark}{theorem}
\newtheorem{remark}[remark]{Remark}
\aliascntresetthe{remark}
\theoremstyle{claim}
\newaliascnt{claim}{theorem}
\newtheorem{claim}[claim]{Claim}
\aliascntresetthe{claim}

\newcommand{\X}{\mathcal X}
\newcommand{\Z}{\mathcal Z}
\newcommand{\err}{\operatorname{err}}
\newcommand{\disc}{\operatorname{disc}}
\newcommand{\VC}{\operatorname{VC}}

\newcommand{\MB}{\operatorname{MB}}
\newcommand{\Wb}{W^{(i \bmod 2)}}
\newcommand{\Wmb}{W^{((i-1) \bmod 2)}}
\newcommand{\LabelBoost}{\textsc{LabelBoostProcedure}}

\newcommand{\Tableofcontents}{
	\thispagestyle{empty}
	\pagenumbering{gobble}
	\clearpage
	\tableofcontents
	\thispagestyle{empty}
	\clearpage
	\pagenumbering{arabic}
}

\newcommand{\authnote}[2]{}

\def\cA{{\cal A}}

\def\cC{{\cal C}}
\def\cD{{\cal D}}

\def\cH{{\cal H}}

\def\cX{{\cal X}}
\def\cY{{\cal Y}}
\def\cZ{{\cal Z}}

\newcommand{\remove}[1]{}

\newcommand{\eps}{\varepsilon}

\title{Private Component-by-Component Learning}
\author{
	Dvir Karni\thanks{Department of Computer Science and Artificial Intelligence, Bar-Ilan University. E-mails: \texttt{dkarni07@gmail.com}, \texttt{eliad.tsfadia@biu.ac.il}.}
	\and
	Eliad Tsfadia$^\ast$
}
\date{}

\begin{document}
\maketitle

\begin{abstract}
We study differentially private learning problems in the realizable setting, where a hypothesis is specified by $k$ components. A direct iteration of private component learners is obstructed by a simple difficulty: an approximate choice of the next component may destroy exact realizability of the labeled sample, even when the next component is locally accurate. We restore realizability using the LabelBoost procedure of \citet{beimel2015learning} and recycle data through two alternating reservoirs. The resulting learner, for a target privacy $\eps$, pays only $\widetilde O(\sqrt{k}/\varepsilon)$ overhead relative to the active sample requirement of a single component learning step at target accuracy $\Theta(\alpha/k)$. For learning $d$-dimensional halfspaces over a finite coordinate grid of size $L$, exact realizability makes the direct component-depth objective quasi-concave. Instantiating the framework with the IPConcave algorithm of \citet{nissim2026quasiconcave} and with the quasi-concave optimizer of \citet{cohen2022qc} yields a realizable sample complexity of \( \tilde{O}\left(\frac{1}{\eps \alpha}\cdot \min\{d^{2.5} \log^*L, \:\: d^{2.5} + d^{1.5} 2^{\log^*L}\}\right), \) which improves on the previously known bound of \( \tilde{O}\left(\frac{1}{\eps \alpha}\cdot \min\{\frac{1}{\alpha}\cdot  d^{5.5} \log^*L, \:\: d^{2.5} 2^{\log^*L}\}\right). \) We also apply the framework to Boolean compositions: given proper private learners for classes $\mathcal H_1,\ldots,\mathcal H_k$, we obtain a proper private learner for $G(\mathcal H_1,\ldots,\mathcal H_k)$ for any fixed Boolean function $G:\{0,1\}^k\to\{0,1\}$. The resulting sample bound incurs a $\widetilde O(\sqrt{k}/\varepsilon)$ overhead relative to the maximum of the component learning costs at accuracy $\Theta(\alpha/k)$ and the relevant VC-dimension term. Compared with the closure theorem of \citet{alon2020closure}, this reduces the overhead on a common component sample bound from $\widetilde O(k/\varepsilon)$ to $\widetilde O(\sqrt{k}/\varepsilon)$.
\end{abstract}

\Tableofcontents

\section{Introduction}

Statistical analysis seeks to extract information about a population without
exposing the individuals whose records make the analysis possible.
Differential privacy provides a mathematical formulation of this goal by
requiring that changing one person's record have only a limited effect on
the distribution of the released result \citep{DMNS06}. Formally, a
randomized algorithm $A$ is $(\varepsilon,\delta)$-differentially private
if, for every pair of datasets $S,S'$ differing in one record and every
event $E$,
\[
    \Pr[A(S)\in E]\le e^{\varepsilon}\Pr[A(S')\in E]+\delta.
\]
The parameters $\varepsilon$ and $\delta$ control the strength of the
guarantee. Throughout this work, we study approximate differential privacy,
where $\delta>0$ is small, and ask how much data is needed to learn while
providing this protection.

\subsection{Private Learning}

Privacy is particularly relevant to learning: training examples may contain
medical, financial, or other sensitive information about individuals, while
the resulting model may be released or used to make predictions for many users.
A private PAC learner combines the usual requirement of accurate prediction
with differential privacy for its training sample
\citep{Kasiviswanathan2011}. In the realizable setting considered here, the
examples are labeled by an unknown concept in a class $\mathcal C$; the
learner must output a hypothesis with population error at most $\alpha$,
with probability at least $1-\beta$.

How much harder does this privacy requirement make learning? Non-private
sample complexity is characterized by the VC dimension, up to the
dependence on accuracy and confidence \citep{blumer1989learnability}.
For approximate private learning, a comparably tight quantitative
characterization remains unknown. There is a qualitative connection to
online learning: private learnability is equivalent to finite Littlestone
dimension, the parameter governing the optimal realizable online mistake
bound \citep{AlonLMM19,bun2020equivalence}. This connection, however, leaves a large
gap between the known sample-complexity bounds, which can be separated by
a tower of exponentials. It also shows that even
thresholds on an infinite ordered domain, whose VC dimension is one, need
not be privately learnable \citep{AlonLMM19}. We refer to \citep{nissim26a} for an extended survey.

Finite domains give a different perspective. Every finite concept class
can be privately learned using $O(\log|\mathcal C|)$ examples when accuracy,
confidence, and privacy are fixed \citep{Kasiviswanathan2011}, but this may
greatly exceed its VC dimension. As emphasized by \citet{nissim26a}, known
bounds for several natural finite classes are much closer to the non-private
baseline, and we lack examples establishing a large separation of the sort
suggested by the generic upper bounds. Determining the true cost of privacy
is therefore a quantitative question even when a private learner is known
to exist.

Halfspaces, which classify points using a linear decision boundary, are a
fundamental class for studying the sample cost of privacy. Their non-private
sample complexity is well understood, but how many additional examples
privacy requires remains unresolved. For halfspaces over a
$d$-dimensional domain over a finite coordinate grid of size $L$,
the non-private sample complexity is $\Theta(d)$ at constant accuracy and
confidence. The VC lower bound and the private threshold lower bound give
$\Omega(d+\log^*L)$ examples as a necessary requirement, in the standard
regime of small $\delta$ \citep{blumer1989learnability,bun2015thresholds,AlonLMM19}.
For proper learners, which must output a halfspace, the stronger lower
bound $\Omega(d\log^*L)$ follows from \citet{bun2015thresholds}.
On the other hand, the previous private upper bounds give
\[
    \widetilde O\!\left(
        \min\left\{d^{2.5}2^{\log^*L},\;
                    d^{5.5}\log^*L\right\}
    \right),
\]
where this comparison fixes accuracy, confidence, and $\varepsilon$ and
suppresses logarithmic factors, including the dependence on $1/\delta$.
The first bound follows from the geometric constructions of
\citet{BeimelMNS19,kaplan2020halfspaces}, using the improved quasi-concave
optimizer of \citet{cohen2022qc}; the second is due to
\citet{nissim2026quasiconcave}. Both leave a substantial gap in the dimension $d$ compared to the lower bounds. 

\subsection{Private Component-by-Component Learning}

Many learning problems admit a representation of a hypothesis by several
coordinates or components. Although choosing all components simultaneously
may be difficult, choosing one component while allowing the others to vary
can be more tractable. This suggests building the hypothesis sequentially.
The challenge is that a private choice is typically only approximately
optimal: after fixing it, the remaining components may no longer admit any
completion consistent with the sample. Thus, even if the original learning
problem is realizable, the intermediate problems need not be. We describe
two settings in which this loss of realizability creates additional costs.

\paragraph{Halfspaces.}

The coordinate-by-coordinate approach to private halfspace learning was developed through
private center points by \citet{BeimelMNS19} and a direct reduction to
linear feasibility by \citet{kaplan2020halfspaces}, while
\citet{nissim2026quasiconcave} subsequently improved its dependence on the
domain size. In linear feasibility, the dataset is a collection
$S=((a_j,b_j))_{j=1}^n$ of constraints $\langle a_j,z\rangle\ge b_j$
on $z\in\mathbb R^d$. Its depth function is
\[
    \operatorname{depth}_S(z)
    :=\bigl|\{j\in[n]:\langle a_j,z\rangle\ge b_j\}\bigr|.
\]
Given a feasible system, a private algorithm seeks a point satisfying most
of its constraints. A labeled halfspace example induces a linear constraint
on the unknown separating hyperplane's parameters, so this task can be used
to learn a halfspace (see \Cref{app:halfspace-linear-feasibility} for the
formal reduction).

After fixing a prefix $z_{<i}$, the natural objective for coordinate $i$ is
\begin{align}\label{eq:ith-depth-objective}
    Q_i^S(t)
    :=\max_{z_{i+1},\ldots,z_d\in\mathbb R}
       \operatorname{depth}_S(z_{<i},t,z_{i+1},\ldots,z_d).
\end{align}
It measures how many constraints a completion can satisfy. A
one-dimensional function $Q$ is quasi-concave if
$Q(t)\ge\min\{Q(u),Q(v)\}$ whenever $u\le t\le v$.
This property makes private coordinate optimization possible with a very
weak dependence on the search-domain size. However, an approximate choice
may leave no fully feasible completion, and subsequent depth objectives
can fail to be quasi-concave.\footnote{The example of
\citet{kaplan2020halfspaces} consists of the constraints
$z_2\ge z_1$ and $z_2\le-z_1$. The original system is feasible, but after
fixing $z_1=1$, the remaining objective satisfies
$Q_2^S(-1)=Q_2^S(1)=1$ and $Q_2^S(0)=0$, violating quasi-concavity.}

The linear-feasibility algorithms of \citet{kaplan2020halfspaces} and
\citet{nissim2026quasiconcave} therefore use the convexified depth function
$\operatorname{cdepth}_S$, whose superlevel sets are the convex hulls of
the corresponding depth superlevel sets. Its coordinate objectives remain
quasi-concave under every prefix. The price is a loss when converting back
to actual constraint satisfaction:
\[
    \operatorname{depth}_S(z)
    \ge (d+1)\operatorname{cdepth}_S(z)-dn
    \qquad\text{\citep{kaplan2020halfspaces}}.
\]
Thus a convexified depth of $(1-\eta)n$ can become a depth of $(1-(d+1)\eta)n$.
Both algorithms must therefore optimize
convexified depth to an error about $d$ times smaller than the desired
error in actual depth.

The two approaches differ in how they optimize each coordinate.
The general quasi-concave optimizer of \citet{BNS13b}, used by
\citet{kaplan2020halfspaces}, has a $2^{O(\log^*N)}$ dependence on a
finite search domain of size $N$. \citet{cohen2022qc} sharpen this to
$\widetilde O(2^{\log^*N})$ additive error at constant privacy.
In contrast, \citet{nissim2026quasiconcave} use additional structure:
the objective can be approximated by a random subsample. Their
\textsc{IPConcave} procedure partitions the data, computes a maximizer on
each block, and privately selects an interior point of these maximizers.
Uniform approximation makes the block maximizers good for the full sample,
and quasi-concavity makes every point between them good as well. The
interior-point algorithm of \citet{cohen2022qc}, unlike a general quasi-concave optimizer, requires only
$\widetilde O(\log^*N)$ block maximizers at constant privacy.

The approach of \citet{nissim2026quasiconcave} thus obtains an essentially
optimal dependence on the domain size $L$, but incurs a further cost in the dimension:
once exact feasibility has been lost, its analysis uses uniform additive
VC approximation, requiring blocks of size $\widetilde O(d/\rho^2)$ for
approximation error $\rho$. To achieve final error $\alpha$, the loss
over $d$ coordinates and the conversion from convexified depth to depth
together require $\rho=\Theta(\alpha/d^2)$. The resulting block size is
$\widetilde O(d^5/\alpha^2)$.\footnote{To maintain a fixed overall privacy
guarantee, advanced composition across the $d$ iterations requires
reducing the privacy parameter of each iteration by a factor of
$\widetilde O(\sqrt d)$. This increases the number of blocks needed by
the interior-point algorithm by the same factor, yielding the
$d^{5.5}$ dependence in the final sample complexity.}

Our key observation is that maintaining realizability throughout
the coordinate construction would allow two improvements.  First, optimizing
depth directly would eliminate the conversion loss, so that per-coordinate
error $\rho=\Theta(\alpha/d)$ would suffice. Second, each block would
remain feasible, and we would only need every solution satisfying all its
constraints to satisfy all but a $\rho$ fraction of the full sample.
This realizable VC guarantee requires just $\widetilde O(d/\rho)$
constraints, giving blocks of size $\widetilde O(d^2/\alpha)$.
Together, these changes save a factor of $d^3$ at constant accuracy;
our halfspace application below realizes this improvement.

\paragraph{Boolean composition.}
Let $G:\{0,1\}^k\to\{0,1\}$ be fixed and let
$\mathcal H_1,\ldots,\mathcal H_k$ be concept classes on a common domain.
The composed class is
\[
    G(\mathcal H_1,\ldots,\mathcal H_k)
    :=\{x\mapsto G(h_1(x),\ldots,h_k(x)):h_i\in\mathcal H_i\}.
\]
It includes, for example, intersections, unions, and majority votes.
Given private learners for the component classes, we would like to learn
the composition. The observed labels specify only the value of $G$, not
the individual component labels, so they cannot be supplied directly to
a realizable learner for $\mathcal H_i$.

For previously selected components $h_1,\ldots,h_{i-1}$, the natural
objective for a candidate $h\in\mathcal H_i$ is its best completion error,
which we seek to minimize:
\begin{align}\label{eq:ith-comp-objective}
    Q_i^S(h)
    :=\min_{c_{i+1}\in\mathcal H_{i+1},\ldots,c_k\in\mathcal H_k}
       \operatorname{err}_S
       \bigl(G(h_1,\ldots,h_{i-1},h,c_{i+1},\ldots,c_k)\bigr).
\end{align}
Here $\operatorname{err}_S(f)$ is the fraction of examples in $S$ that
$f$ misclassifies.
\citet{alon2020closure} use a relabeling transformation based on
LabelBoost \citep{beimel2015learning}: select a component labeling with
small completion error, then learn that labeling with the corresponding
private learner. This makes the local learning call realizable. It does
not, however, restore exact consistency of the selected global prefix
with the original labels. Their construction uses a separate sample for
each component and additive generalization bounds for the accumulated
completion error. Consequently, its sample bound contains a sum of the
component learning costs and a term of order
$\widetilde O(k^3\VC(G(\mathcal H_1,\ldots,\mathcal H_k))/\alpha^2)$
at constant privacy \citep[Theorem~29]{alon2020closure}.
When the component learners have a common sample requirement $m$, using
$k$ disjoint samples incurs a cost of $km$.
Our construction draws samples from shared reservoirs and combines
subsampling amplification with advanced composition, reducing the
multiplier on the component learning cost to $\widetilde O(\sqrt{k})$
at constant privacy. Doing so requires accounting for the dependence
introduced by repeated relabeling, as well as preserving the utility of
later component choices.

\subsection{Our Results}
We give a general learner that preserves realizability while constructing
a hypothesis component by component.
After each component choice, we relabel the next active sample so
that it is realizable by a concept extending the chosen prefix.
The component oracles therefore need to guarantee both privacy and
utility only on samples realizable by the current prefix class,
while the resulting learner is private on arbitrary datasets.

Our framework provides two main advantages.
First, we exploit these weaker requirements in our applications
to replace uniform additive VC bounds by realizable bounds,
reducing the corresponding dependence on inverse accuracy
from quadratic to linear.
This is particularly useful when errors accumulate across components
and each step requires accuracy $\Theta(\alpha/k)$.
Second, we perform repeated LabelBoost relabeling
\citep{beimel2015learning} while reusing data across rounds,
with a total sample overhead of
$\widetilde O(\sqrt{k}/\varepsilon)$ over the oracle and
generalization requirements of one round.
Our extended privacy analysis accounts for the dependencies
introduced by relabeling, allowing us to obtain this guarantee
without allocating a separate sample to every component.

More formally, write each concept as
$c_{y_1,\ldots,y_k}$ and let $\mathcal C(y_{\le i})$ be the concepts
whose representations extend the chosen prefix. We assume oracles
$A_1,\ldots,A_k$, where $A_i$ receives a valid prefix $y_{<i}$ and a
labeled sample $S$, and outputs a next component $y_i$. On every input,
the output must extend the prefix to a valid representation. For each
fixed prefix, $A_i$ must be $(\varepsilon_A,\delta_A)$-differentially
private between neighboring samples realizable by $\mathcal C(y_{<i})$.
If $S$ is realizable by this class and contains at least
$M_{A}(\tau,\beta_A,\varepsilon_A,\delta_A)$ examples, then, with
probability at least $1-\beta_A$, the extended class
$\mathcal C(y_{\le i})$ must contain a hypothesis with empirical error
at most $\tau$ on $S$. The algorithm need not output this completion or
preserve exact consistency (See \Cref{ass:component-oracle} for the formal assumed properties).

In the informal statements below, $\widetilde O$ suppresses logarithmic
factors in the dimension, the number of components, accuracy, confidence,
and privacy parameters, and lower-order factors in the iterated logarithms.
When we abbreviate a component sample requirement by $M_A(\tau)$, the
oracle's privacy parameter $\varepsilon_A$ is constant, and its failure
probability and $\delta_A$ are of order $\beta/k$ and $\delta/k$,
respectively. The dependence on these parameters remains part of $M_A$.

\begin{theorem}[Informal version of \Cref{thm:carry-forward-main}]
\label{thm:intro-generic}
Let $\mathcal C$ have VC dimension $v$ and a representation by $k$
components. Suppose private component oracles as above have common
sample requirement $M_A(\tau)$. Then, for
$\alpha,\beta,\varepsilon,\delta\in(0,1)$ with $\delta\le\beta$, there
is a proper $(\varepsilon,\delta)$-differentially private realizable
$(\alpha,\beta)$-PAC learner for $\mathcal C$ using
\[
    \widetilde O\!\left(
        \frac{\sqrt{k}}{\varepsilon}
        \max\left\{\frac{kv}{\alpha},\;
                    M_A\!\left(\frac{\alpha}{40k}\right)\right\}
    \right)
\]
examples. Its privacy guarantee holds on arbitrary input datasets.
\end{theorem}

The bound separates the cost of one accurate component step from the cost
of combining $k$ steps. The active sample must support the oracle at
accuracy $\Theta(\alpha/k)$ and generalization within $\mathcal C$.
The total sample size is only a
$\widetilde O(\sqrt{k}/\varepsilon)$ multiple of this requirement.

After each component choice, we use the LabelBoost algorithm of
\citet{beimel2015learning} to relabel the next active
sample by a hypothesis in the extended prefix class. The next oracle
therefore again sees an exactly realizable problem, while the disagreement
introduced by each relabeling is small enough to accumulate over all $k$
rounds. To control privacy without allocating a disjoint sample to every
round (as done by \citet{alon2020closure}), we draw the active samples
from only two alternating reservoirs. An
extended LabelBoost analysis bounds the privacy cost of two consecutive
rounds, allowing subsampling amplification and advanced composition across
repeated uses of each reservoir.

\subsubsection{Applications}

\paragraph{Halfspaces.}
We consider halfspaces over \(\cX^d\), where
\(\cX\subseteq\mathbb R\) is a finite coordinate grid of size
\(L:=|\cX|\).
We apply the framework with parameters \(k=d+1\) and
VC dimension \(v=d+1\).
A key observation is that the direct coordinate-depth objective in
\Cref{eq:ith-depth-objective} is quasi-concave whenever the current prefix
has a fully feasible completion
(\Cref{lem:realizable-depth-quasiconcave-appendix}). In the underlying
linear-feasibility problem, $A_i$ receives the current prefix and constraint
sample $S$, forms this objective, and privately selects a coordinate $z_i$
with $Q_i^S(z_i)\ge(1-\tau)|S|$. The selected prefix therefore admits a
completion satisfying all but a $\tau$ fraction of the constraints.
Restoring realizability between rounds lets every $A_i$ optimize depth
directly, avoiding the loss from convexification.

\begin{theorem}[Informal version of \Cref{thm:halfspace-framework}]
\label{thm:intro-halfspace-framework}
Suppose the coordinate oracles for $d$-dimensional halfspaces satisfy
the component-oracle requirements with common sample requirement
$M_A(\tau)$. Then there is a proper
$(\varepsilon,\delta)$-differentially private realizable
$(\alpha,\beta)$-PAC learner using sample complexity
\[
    \widetilde O\!\left(
        \frac{\sqrt d}{\varepsilon}
        \max\left\{\frac{d^2}{\alpha},\;
           M_A\!\left(\frac{\alpha}{80(d+1)}\right)\right\}
    \right).
\]
\end{theorem}

Let $N$ be the largest finite-coordinate domain size used by the
reduction.
We instantiate $A_i$ in two ways. The first applies the general
quasi-concave optimizer of \citet{cohen2022qc} directly to $Q_i^S$;
achieving depth loss at most $\tau|S|$ requires
$\widetilde O(2^{\log^*N}/\tau)$ constraints at constant privacy.
The second uses \textsc{IPConcave} of \cite{nissim2026quasiconcave}: randomly partition $S$ into blocks,
find a fully feasible completion of the prefix on each block, and apply
a private interior-point algorithm to these completions' $i$th coordinates.
Realizable generalization makes each selected coordinate good for the full
sample, and quasi-concavity ensures that every point between them is good
as well. Blocks of size $\widetilde O(d/\tau)$ and
$\widetilde O(\log^*N)$ block solutions suffice, giving a cost of
$\widetilde O(d\log^*N/\tau)$.
Combining these optimization costs with the framework's generalization
requirement, using \(\log^\ast N\approx\log^\ast L\) with the
additional factors absorbed by \(\widetilde O\), yields the following
learning bounds.

\begin{theorem}[Informal versions of
{\Cref{thm:private-realizable-halfspaces,thm:private-realizable-halfspaces-ipconcave}}]
\label{thm:intro-halfspaces}
For $d$-dimensional halfspaces over a finite coordinate grid of size $L$, there is a proper
$(\varepsilon,\delta)$-differentially private realizable
$(\alpha,\beta)$-PAC learner with sample complexity
\[
    \widetilde O\!\left(
        \frac{1}{\varepsilon\alpha}
        \min\left\{
            d^{2.5}+d^{1.5}2^{\log^*L},\;
            d^{2.5}\log^*L
        \right\}
    \right).
\]
\end{theorem}

Namely, the \textsc{IPConcave} instantiation reduces the dimension dependence
from $d^{5.5}$ to $d^{2.5}$ while retaining the $\log^*L$ dependence.
The general quasi-concave instantiation reduces the coefficient of
$2^{\log^*L}$ from $d^{2.5}$ to $d^{1.5}$, with an additional
$d^{2.5}$ term. All these gains come from restoring realizability before each
coordinate optimization.

\paragraph{Boolean composition.}

Suppose each $\mathcal H_i$ has a proper private realizable learner $B_i$.
Given a prefix $h_1,\ldots,h_{i-1}$ and a sample $S$ realizable by its
restricted composition class, we define $A_i$ to use the relabel-and-learn construction
of \citet{alon2020closure}: First, the LabelBoost relabeling step selects
$\widetilde h_i\in\mathcal H_i$ with small completion error
$Q_i^S(\widetilde h_i)$ (defined in \Cref{eq:ith-comp-objective}), using the exponential mechanism over the
labelings of $S$ induced by $\mathcal H_i$. It then replaces the sample
labels by $\widetilde h_i$'s predictions and runs the empirical version
of $B_i$ on this realizable sample, returning its output $h_i\in\mathcal H_i$.
Only $h_i$ is released by the oracle; the relabeling is an internal step.

If $Q_i^S(\widetilde h_i)\le\tau/2$ and $h_i$ disagrees with
$\widetilde h_i$ on at most a $\tau/2$ fraction of $S$, then replacing
$\widetilde h_i$ by $h_i$ in a good completion increases its error by at
most $\tau/2$. Thus $Q_i^S(h_i)\le\tau$, as required.
Properness keeps the output in the component class, while the outer
learner restores realizability of the global prefix before the next round
and reuses the reservoirs across rounds.

\begin{theorem}[Informal version of \Cref{thm:private-realizable-boolean}]
\label{thm:intro-boolean}
Let $G:\{0,1\}^k\to\{0,1\}$ and
$\mathcal C=G(\mathcal H_1,\ldots,\mathcal H_k)$. Suppose each $\mathcal H_i$ has a
proper private realizable PAC learner with sample requirement $m_i(\tau)$,
at constant $\varepsilon_i$ and failure and approximate-privacy parameters
of order $\beta/k$ and $\delta/k$. Set
\[
    v_*:=\max\left\{\VC(\mathcal C),\;
                    \max_{i\in[k]}\VC(\mathcal H_i)\right\}.
\]
Then $\mathcal C$ has a proper $(\varepsilon,\delta)$-differentially
private realizable $(\alpha,\beta)$-PAC learner with sample complexity
\[
    \widetilde O\!\left(
        \frac{\sqrt{k}}{\varepsilon}
        \max\left\{\frac{kv_*}{\alpha},\;
                    \max_{i\in[k]}m_i\!\left(\frac{\alpha}{80k}\right)
            \right\}
    \right).
\]
\end{theorem}

For a common component sample bound, the multiplier on that bound improves
from $\widetilde O(k/\varepsilon)$ in the closure construction of
\citet{alon2020closure} to $\widetilde O(\sqrt{k}/\varepsilon)$.
Our dimension term is $\widetilde O(k^{3/2}v_*/(\varepsilon\alpha))$,
using small-error generalization after each realizability repair.
We remark that this comparison concerns proper component learners, while the theorem of
\citet{alon2020closure} also permits improper learners.
Our approach may extend to improper learners as well; we focus on proper
learners here to simplify the realizability and generalization arguments.

\subsection{Related Work}

Private threshold learning and the interior-point problem provide the
one-dimensional primitives underlying private geometric learning.
Following the early quasi-concave optimizer of \citet{BNS13b} and the
lower bounds of \citet{bun2015thresholds} and \citet{AlonLMM19},
\citet{kaplan2020thresholds} reduced the dependence on the domain size
from exponential to polynomial in its iterated logarithm.
\citet{cohen2022qc} subsequently obtained nearly optimal bounds for
thresholds and proved a separation from general quasi-concave
optimization. The center-point and linear-feasibility constructions of
\citet{BeimelMNS19} and \citet{kaplan2020halfspaces} extend these primitives to
halfspaces, while \citet{nissim2026quasiconcave} exploit approximation by
subsamples to bypass the general quasi-concave lower bound for geometric
objectives. Our halfspace results build on these constructions, with the
improvements in dimension described above.\footnote{\citet{huang2025halfspaces}
claim a stronger sample bound for halfspaces, but
\citet{nissim2026quasiconcave} give a counterexample to the privacy
guarantee of their proposed algorithm.}

A complementary line of work emphasizes computational efficiency.
\citet{nguyen2020margin} give efficient private learners for large-margin
halfspaces with sample complexity independent of the ambient dimension.
More recently, \citet{kaplan2025differentially} develop efficient private
algorithms for linear-algebraic tasks, including linear inequalities,
and obtain halfspace learners as an application. Our results concern
sample complexity without a margin assumption; computational efficiency
remains a separate challenge.

Our general framework also belongs to the literature on transformations
between private learning problems. LabelBoost was introduced by
\citet{beimel2015learning} to reduce the number of labeled examples needed
for private learning, and its relabeling transformation also yields a
reduction from agnostic learning to proper realizable learning.
\citet{alon2020closure} extend this reduction to improper learners and
prove closure under arbitrary Boolean composition, alongside closure
results for online learning. We use a more efficient version of their relabel-and-learn construction
for our Boolean component oracles, and improve the cost of combining
proper component learners by preserving realizability across rounds.

Specialized constructions exploit additional structure:
\citet{sadigurschi2021rectangles} obtain nearly linear dependence
on dimension for axis-aligned rectangles by deleting exposed
records, avoiding generic composition costs.
\citet{bun2026decisionlists} give an efficient private PAC learner
for decision lists with nearly non-private sample complexity,
and a private Winnow algorithm for online learning of large-margin
halfspaces. Our framework complements these specialized constructions by
studying how realizability and data reuse can reduce the sample cost of
sequential component choices.
\subsection{Open Questions}

We have given a framework for private component-by-component learning that
restores realizability after each component choice and shares samples
across rounds with a $\widetilde O(\sqrt{k}/\varepsilon)$ overhead.
Its applications improve the dependence on the dimension for halfspaces
and on the number of components for Boolean composition.

Our improved upper bounds for halfspaces still exceed the known lower
bounds: $\Omega(d\log^*L)$ for proper learners
\citep{bun2015thresholds}, and
$\Omega(d+\log^*L)$ for learners that may be improper
\citep{AlonLMM19}.
Closing these gaps, particularly the polynomial gap in $d$, would clarify
how much of the current overhead is inherent to privacy.

Furthermore, our construction is primarily
information-theoretic: selecting representatives of labelings, searching
finite coordinate domains, and finding good completions may require
superpolynomial time.
\citet{kaplan2025differentially} give efficient private algorithms for
linear-algebraic problems and, as an application, halfspace learning, with
sample complexity polynomial in the dimension and polylogarithmic in the
grid size $L$. It remains open whether efficient halfspace learners can
achieve the much weaker $\log^*L$ dependence while retaining a small
dependence on $d$. More
generally, it would be useful to identify component representations and
relabeling procedures for which our framework can be implemented
efficiently.

\subsection{Paper Organization}

We give an overview of our main techniques in \Cref{sec:technique}.
\Cref{sec:preliminaries} collects the technical preliminaries, including
the properties of \textsc{LabelBoost} used in our analysis.
In \Cref{sec:component-framework}, we formalize component-by-component
learning and the requirements on the component oracles.
\Cref{sec:generic-learner} presents our generic
\textsc{CarryForwardLearner} algorithm and its main guarantee, with
the privacy and utility analyses in \Cref{sec:privacy,sec:utility},
respectively. In \Cref{sec:applications}, we instantiate the framework
for halfspace learning and closure under Boolean composition.
Additional proofs appear in
\Cref{app:depth-quasiconcavity,app:local-proper-blackbox,app:halfspace-linear-feasibility}.
These establish quasi-concavity of the depth function, verify the
construction of component oracles from local proper learners, and
provide the details of the reduction between halfspace learning and
linear feasibility, including the perturbation and finite-domain
arguments.

\section{Our Technique}\label{sec:technique}

In this section, we give a proof overview of our main result, \Cref{thm:intro-generic}.

\subsection{The Setting}

Let $\cC$ be a concept class of VC dimension $v$, whose classifiers have
the form $c_{y_1,\ldots,y_k}:\cX\to\{0,1\}$, with components
$y_i\in\cY_i$. For a prefix
$y_{<i}=(y_1,\ldots,y_{i-1})$, let $\cC(y_{<i})$ denote the concepts
whose representations extend that prefix. We are given component oracles
$A_1,\ldots,A_k$. On input a valid prefix $y_{<i}$ and a labeled sample
$S\in\cZ^m$, where $\cZ=\cX\times\{0,1\}$, oracle $A_i$ outputs a
valid next component $y_i$.

Our goal is to combine these oracles into a proper
$(\varepsilon,\delta)$-differentially private learner for $\cC$.
Given i.i.d.\ examples from a distribution realizable by $\cC$, the
learner should output a classifier $c_{y_1,\ldots,y_k}$ with population
error at most $\alpha$, with probability at least $1-\beta$.
Its privacy guarantee must hold on arbitrary input datasets.
We seek to achieve this using only a $\widetilde O(\sqrt{k}/\varepsilon)$
overhead relative to the sample size needed for one component step and
the accompanying generalization guarantee.

The oracle's privacy and accuracy guarantees are required only when $S$
is realizable by $\cC(y_{<i})$. More precisely, for every fixed prefix,
$(\varepsilon_A,\delta_A)$-privacy holds between neighboring samples that are both realizable by
this restricted class, possibly by different hypotheses
(\Cref{def:realizable-dp}). For accuracy, with high probability the
extended class $\cC(y_{\le i})$ contains a completion with empirical
error at most $\tau$ on $S$. The oracle need not identify that completion.
These requirements are formalized in \Cref{ass:component-oracle}.

A direct and naive attempt would run the oracles sequentially on the original
sample. However, even when the sample is initially realizable, an
approximately accurate component choice can leave no consistent
completion of the selected prefix. The next oracle then receives an
input outside its promised domain, so neither its privacy nor its utility
guarantee applies. We must therefore restore exact realizability after
each component choice, while controlling the error introduced by doing so.

\subsection{Restoring Realizability via LabelBoost}

We use the LabelBoost procedure of \citet{beimel2015learning}. Given a
labeled sample $S$, an unlabeled sample $T$, and a concept class
$\mathcal C$, the procedure considers the labelings of $U(S)\cup T$
induced by $\mathcal C$, keeping one representative hypothesis for each
labeling. It selects a representative using the exponential mechanism,
giving greater weight to hypotheses that agree with more of the labels
in $S$, and relabels both samples according to it. Here $U(S)$ denotes
the sample with its labels removed, and $U(T)^h$ denotes $T$ labeled by $h$.
The resulting sample is exactly realizable by $\mathcal C$, even if the
original sample was not. If $\mathcal C$ contains a hypothesis of small
empirical error, LabelBoost introduces only a small additional error.

The selected hypothesis need not itself be private: its candidate set
depends on the input points. Nevertheless, \citet{beimel2015learning}
show that LabelBoost followed by an $(\varepsilon,\delta')$-private mechanism
on the relabeled data is $(\varepsilon+3,4e\delta')$-private
(\Cref{lem:labelboost-privacy-blackbox}). Thus, a continuation with
constant privacy cost gives constant privacy for the combined procedure;
subsampling from a larger dataset can reduce this to a small privacy
parameter. Our extended analysis also allows the continuation's privacy
guarantee to hold only on realizable inputs
(\Cref{lem:labelboost-realizable-continuation}).

The challenge is to perform $k$ such relabel-and-learn steps without
allocating a new large dataset to every step. In their construction for
Boolean composition, \citet{alon2020closure} use a separate sample for
each component. When the component sample requirements are comparable,
this incurs a factor of $k$. We instead reuse two reservoirs, with a
privacy analysis that accounts for the intermediate relabeled samples
passed between rounds.

\subsection{Our CarryForwardLearner Algorithm}

\textsc{CarryForwardLearner} (\Cref{alg:carry-forward}) partitions its
$n$ input records, independently of their values, into three blocks of
comparable size: an initialization block $W^{\mathrm{init}}$ and two
reservoirs $W^{(0)},W^{(1)}$. The active sample size is
\[
    m=\Theta\!\left(
        \frac{\varepsilon n}{\sqrt{k\log(2/\delta)}}
    \right),
\]
with a sufficiently small constant of proportionality. To initialize, we
draw a size-$m$ subsample $R_0$ from $W^{\mathrm{init}}$ and apply
LabelBoost with class $\cC$. If $h_0\in\cC$ is the internally selected
hypothesis, we set $S_0=U(R_0)^{h_0}$. This preliminary relabeling ensures
that even an arbitrary input dataset yields a realizable input to $A_1$.

At iteration $i$, the learner has a prefix $y_{<i}$ and an active sample
$S_{i-1}$ labeled by some $h_{i-1}\in\cC(y_{<i})$. It proceeds as follows:
\begin{enumerate}
    \item Compute $y_i\leftarrow A_i(y_{<i};S_{i-1})$ and extend the
    prefix to $y_{\le i}$. If $i=k$, return $c_{y_1,\ldots,y_k}$.

    \item Sample a fresh size-$m$ subset
    $T_i\subseteq W^{((i-1)\bmod 2)}$ uniformly without replacement.
    After the first iteration, this is the reservoir opposite to the one
    carrying the current active sample.

    \item Apply LabelBoost to the labeled sample $S_{i-1}$ and the
    unlabeled sample $U(T_i)$, using the extended class $\cC(y_{\le i})$.
    Let $h_i\in\cC(y_{\le i})$ be its internally selected hypothesis.

    \item Set $S_i=U(T_i)^{h_i}$ and continue to the next iteration.
\end{enumerate}
The reservoirs are kept unchanged, so subsets sampled in different rounds
may overlap. Only the final classifier is released. Note that by construction,
$S_i$ is exactly realizable by $\cC(y_{\le i})$ in every execution,
regardless of whether any preceding oracle met its accuracy guarantee
(\Cref{lem:iterated-realizability}).

\paragraph{Privacy analysis sketch.}

We use a constant oracle privacy parameter $\varepsilon_A$.

We first consider a single round with a fixed incoming prefix $y_{<i}$
and a fixed next carrier $T_i$. The round first computes
$y_i\leftarrow A_i(y_{<i};S_{i-1})$. By \Cref{ass:component-oracle},
this oracle call is guaranteed to be $(\varepsilon_A,\delta_A)$-private
only between neighboring samples that are both realizable by
$\cC(y_{<i})$, possibly by different hypotheses. This is the source of
the realizability restriction in the privacy guarantee for the entire round.
For every fixed output $y_i$, our carrier-only LabelBoost lemma
(\Cref{lem:carrier-only-labelboost}) shows that releasing the labeling
$S_i=U(T_i)^{h_i}$ of the fixed carrier costs only
$c_{\mathrm{LB}}=\log(4e)$ in privacy, even on arbitrary neighboring
labeled inputs. Composing this relabeling step with the oracle call
therefore gives $(\varepsilon_A+c_{\mathrm{LB}},\delta_A)$-privacy
for $(y_i,S_i)$ on the same realizable neighboring inputs
(\Cref{claim:second-round-privacy}).
This argument concerns changes in
$S_{i-1}$ with $T_i$ fixed; it does not allow us to release $S_i$
privately with respect to changes in $T_i$, since $S_i$ retains its
underlying records.

To handle that dependence,
we group rounds $i$ and $i+1$ into a
\emph{two-round macro-block}. For a fixed incoming state
$(y_{<i},S_{i-1})$ and fixed reservoir $\Wb$, this mechanism maps
$T_i$ (chosen from the opposite reservoir $\Wmb$) to $(y_i,y_{i+1},S_{i+1})$. The first component $y_i$ is computed
before $T_i$ is accessed. LabelBoost then relabels $T_i$, setting $S_i=U(T_i)^{h_i}$, and the second
round applies the private continuation (the single round computation described above) to this relabeled
sample $S_i$ to obtain $(y_{i+1}, S_{i+1})$. Namely, the map from $T_i$ to $(y_{i+1}, S_{i+1})$ is obtained by applying to the relabeled $T_i$ a computation that is $(\varepsilon_A+c_{\mathrm{LB}},\delta_A)$-differentially private on $\mathcal C(y_{\le i})$-realizable neighboring inputs. Our realizable-input extension of the LabelBoost lemma (\Cref{lem:labelboost-realizable-continuation}) therefore
gives $(O(1),O(\delta_A))$-privacy for the macro-block with respect to
arbitrary neighboring $T_i,T_i'$
(\Cref{lemma:activated-two-round-macro-block}). At its end, the active
sample is carried by the opposite reservoir: the records of $T_i$ are
no longer passed forward.

For each reservoir, it remains to compose roughly $k/2$ macro-blocks.
Each block passes its chosen components and active sample to the next,
and the privacy guarantee above holds for every possible incoming state.
Adaptive composition therefore applies
(\Cref{lem:transcript-determined-boundary}). Before composing, we amplify
the privacy of each block: it accesses the reservoir through a random
subsample containing a fraction
\[
    q=\frac{m}{|W^{(b)}|}
      =O\!\left(\frac{\varepsilon}{\sqrt{k\log(2/\delta)}}\right)
\]
of the reservoir. Subsampling amplification reduces the constant privacy
cost of its macro-block to $O(q)$, with approximate-privacy parameter
$O(q\delta_A)$. Taking $\delta_A$ of order $\delta/k$ and applying
advanced composition over at most $\lceil(k-1)/2\rceil$ activations
gives $(\varepsilon,\delta)$-privacy for that reservoir. The
initialization block is accessed only once and is handled by the same
LabelBoost and amplification argument. Since a neighboring input changes
one record in exactly one block, the complete learner is private on
arbitrary input datasets. The formal analysis is in \Cref{sec:privacy}.

\paragraph{Utility analysis sketch.}

Suppose the input is drawn independently from a distribution realizable
by $h^\star\in\cC$. At the start of round $i$, the active sample is
labeled by $h_{i-1}$. With high probability, oracle $A_i$ leaves a
completion $g_i\in\cC(y_{\le i})$ that disagrees with $h_{i-1}$ on at
most a $\tau$ fraction of that sample. For $i<k$, LabelBoost then
selects $h_i\in\cC(y_{\le i})$ with at most another $\tau$ of empirical
error. Consequently,
\[
    \operatorname{disc}_{U(S_{i-1})}(h_i,h_{i-1})\le 2\tau
\]
on the success event, where $\operatorname{disc}$ denotes the fraction
of points on which two hypotheses disagree
(\Cref{cor:labelboost-approximate-utility}). Thus, restoring realizability
costs only $O(\tau)$ empirical disagreement per round.

The second step is to transfer this bound to the population, despite the
adaptive choice of both hypotheses. We use a small-error VC guarantee
that holds simultaneously for every pair in $\cC$: for a sufficiently
large sample, empirical disagreement at most $2\tau$ implies population
disagreement at most $20\tau$ (\Cref{thm:bns-generalization}). The
underlying points of each active sample are marginally an i.i.d.\ sample
from the population, because the sampled indices are chosen independently
of record values. Different rounds need not be independent. The uniform
guarantee applies to the adaptively chosen hypotheses, and a union bound
over rounds handles reuse of the reservoirs. Both this guarantee and
LabelBoost's utility bound require only $\widetilde O(v/\tau)$ examples. With
$\tau=\Theta(\alpha/k)$, this gives the
$\widetilde O(kv/\alpha)$ requirement on each active sample.
The active sample must also meet the oracle's sample requirement $M_A(\tau)$.

Initialization similarly makes $h_0$ close to $h^\star$. In the final
round, the full prefix has a unique completion
$h_k=c_{y_1,\ldots,y_k}$, so oracle accuracy and generalization directly
bound its disagreement with $h_{k-1}$. Summing the population
disagreements along $h^\star,h_0,\ldots,h_k$ gives error $O(k\tau)$.
In particular, $\tau=\alpha/(40k)$ and failure probabilities of order
$\beta/k$ suffice for total error at most $\alpha$ with probability at
least $1-\beta$ (\Cref{thm:iterated-utility}). Combining the active-sample
requirement with the privacy calibration yields
\[
    n=\widetilde O\!\left(
        \frac{\sqrt{k}}{\varepsilon}
        \max\left\{\frac{kv}{\alpha},\;
            M_A\!\left(\frac{\alpha}{40k}\right)\right\}
    \right),
\]
with the same parameter conventions as in \Cref{thm:intro-generic}.

Throughout both analyses, realizability of every oracle input is an
unconditional invariant, rather than a high-probability utility event.
This is why privacy only on realizable neighbors suffices for the
component oracles, while the resulting learner provides standard
differential privacy on all inputs.

\section{Preliminaries}\label{sec:preliminaries}
\subsection{Notation and PAC learning}

Let \(\mathcal X\) be an instance space, let
\(\mathcal C\subseteq\{0,1\}^{\mathcal X}\) be a binary concept class,
and write \(\mathcal Z:=\mathcal X\times\{0,1\}\).
Samples are finite sequences, with repetitions allowed and counted
with multiplicity. We write \(\circ\) for concatenation.

For a labeled sample
\(S=((x_i,y_i))_{i=1}^n\in\mathcal Z^n\), its unlabeled projection is
\(U(S):=(x_i)_{i=1}^n\).
For an unlabeled sample \(U(S)=(x_i)_{i=1}^n\) and a hypothesis
\(h:\mathcal X\to\{0,1\}\), let
$
U(S)^h:=((x_i,h(x_i)))_{i=1}^n,
$
Thus, \(U(S)^h\) replaces the labels of \(S\) by those assigned by \(h\).
For partially labeled samples, the same notation applies after
discarding the existing labels.

For a distribution \(\mathcal D\) over \(\mathcal Z\), let
\(\mathcal D_X\) denote its marginal over \(\mathcal X\).
We write \(S\sim\mathcal D^n\) and \(U\sim\mathcal D_X^n\) for
\(n\) i.i.d.\ draws from the respective distributions.

For hypotheses \(h,h':\mathcal X\to\{0,1\}\), define
\[
\begin{aligned}
\err_S(h)
&:=\frac{1}{n}\sum_{i=1}^n
\mathbf{1}\{h(x_i)\neq y_i\},
&
\err_{\mathcal D}(h)
&:=\Pr_{(x,y)\sim\mathcal D}[h(x)\neq y],
\\
\disc_{U(S)}(h,h')
&:=\frac{1}{n}\sum_{i=1}^n
\mathbf{1}\{h(x_i)\neq h'(x_i)\},
&
\disc_{\mathcal D_X}(h,h')
&:=\Pr_{x\sim\mathcal D_X}[h(x)\neq h'(x)].
\end{aligned}
\]
We say that \(h\) is consistent with \(S\) if \(\err_S(h)=0\).
The sample \(S\) is \(\mathcal C\)-realizable if some
\(h\in\mathcal C\) is consistent with \(S\).
Similarly, \(\mathcal D\) is \(\mathcal C\)-realizable if
\(\err_{\mathcal D}(h)=0\) for some \(h\in\mathcal C\).
If \(h'\) is consistent with \(S\), then
\(\err_S(h)=\disc_{U(S)}(h,h')\).
Likewise, if \(\err_{\mathcal D}(h')=0\), then
\(\err_{\mathcal D}(h)=\disc_{\mathcal D_X}(h,h')\).

\begin{definition}[Realizable PAC learning]
\label{def:pac-learning}
Let \(\alpha,\beta\in(0,1)\).
An algorithm \(\mathcal A\) is an \((\alpha,\beta,n)\)-PAC learner
for \(\mathcal C\) if, for every \(\mathcal C\)-realizable
distribution \(\mathcal D\) over \(\mathcal Z\), given
\(S\sim\mathcal D^n\), it outputs a hypothesis
\(h:\mathcal X\to\{0,1\}\) satisfying
\[
\Pr\!\left[\err_{\mathcal D}(h)\le\alpha\right]\ge1-\beta.
\]
The probability is over the sample and the internal randomness
of \(\mathcal A\).
The learner is \emph{proper} if it always outputs \(h\in\mathcal C\);
otherwise, it is \emph{improper}.
\end{definition}

\subsection{Differential privacy}
We use substitute-one adjacency on fixed-size datasets: \(S,S'\in\Z^n\) are neighboring, written \(S\simeq S'\), if they differ in one indexed position.\begin{definition}[Differential privacy]
\label{def:dp}
A randomized mechanism \(M\) is
\((\varepsilon,\delta)\)-differentially private if, for all
neighboring datasets \(S\simeq S'\) and all measurable output events
\(E\),
\[
\Pr[M(S)\in E]
\le
e^\varepsilon\Pr[M(S')\in E]+\delta.
\]
Two random variables \(X\) and \(Y\) are
\((\varepsilon,\delta)\)-indistinguishable if, for every measurable
event \(E\),
\(\Pr[X\in E]\le e^\varepsilon\Pr[Y\in E]+\delta\), the same
inequality holds with \(X\) and \(Y\) interchanged.

Thus, differential privacy requires \(M(S)\) and \(M(S')\) to be
\((\varepsilon,\delta)\)-indistinguishable for every neighboring
\(S,S'\).
\end{definition}

\begin{definition}[Private PAC learning {\citep{Kasiviswanathan2011}}]
Let $\cA$ be an algorithm that receives as input a sample
$S=(z_1,\ldots,z_n)$.
We say that $\cA$ is 
$(\alpha,\beta,\varepsilon,\delta,n)$-differentially private PAC learner for a concept class $\cC$ over $\cX$,
using hypothesis class $\cH$ and sample complexity $n$, if:

\begin{itemize}
    \item \textbf{Privacy.}
    $\cA$ is $(\varepsilon,\delta)$-differentially private. (\Cref{def:dp})

    \item \textbf{Utility.}
    $\cA$ is an $(\alpha,\beta,n)$-PAC learner for $\cC$
    using hypothesis class $\cH$. (\Cref{def:pac-learning})
\end{itemize}
We call $\cA$ a \emph{realizable private PAC learner} when its utility
guarantee is required only for distributions realizable by $\cC$.
\end{definition}

\begin{theorem}[Fixed-size subsampling amplification]\label{thm:bbg}\citep[Theorem~9]{balle2018privacy}
Let a mechanism on a size-\(m\) dataset be \((\varepsilon,\delta)\)-DP under substitute-one adjacency.  If its input is obtained by uniformly sampling \(m\) records without replacement from a size-\(M\) reservoir, with \(q=m/M\), then the induced reservoir-level mechanism is
\[
\left(
\log\bigl(1+q(e^{\varepsilon}-1)\bigr),
q\delta
\right)\text{-DP}.
\]
\end{theorem}\begin{theorem}[Adaptive advanced composition]
\label{thm:drv}
\citep[Definition~III.1 and Theorem~III.3]{dwork2010boosting}
Consider an adaptive sequence of \(L\) mechanisms. In each round, the
mechanism to be applied may depend on the outputs of the preceding
rounds. Suppose that, for every round and every fixed realization of
the preceding transcript, the mechanism selected for that round is
\((\varepsilon,\delta)\)-differentially private.
Then, for every \(\delta'>0\), the full composition is
\((\varepsilon_{\mathrm{tot}},L\delta+\delta')\)-differentially private,
where

$$
\varepsilon_{\mathrm{tot}}
\le
\sqrt{2L\log(1/\delta')}\,\varepsilon
+
L\varepsilon(e^{\varepsilon}-1).
$$

\end{theorem}

We also use the standard post-processing property and basic composition
without further restatement.

\begin{definition}[Differential privacy on realizable inputs]
\label{def:realizable-dp}
Fix a concept class
\(\mathcal C\subseteq\{0,1\}^{\mathcal X}\).
A labeled sample
\(S=((x_i,y_i))_{i=1}^m\) is
\(\mathcal C\)-realizable if there exists \(g\in\mathcal C\) such that
\(g(x_i)=y_i\) for every \(i\in[m]\).
A randomized mechanism
\[
A:(\mathcal X\times\{0,1\})^m\to\mathcal R
\]
is \((\varepsilon,\delta)\)-differentially private
\(\mathcal C\)-realizable inputs if, for every pair of neighboring
samples \(S,S'\) that are both \(\mathcal C\)-realizable, possibly by
different hypotheses in \(\mathcal C\), and every event
\(F\subseteq\mathcal R\),
\[
\Pr[A(S)\in F]
\le
e^\varepsilon\Pr[A(S')\in F]+\delta.
\]
\end{definition}

\subsection{Learning-Theoretic Tools} 
\begin{definition}[Trace, growth function, and VC dimension]
\label{def:growth-function}

Let $\mathcal C\subseteq\{0,1\}^{\X}$ and let
$P=\{p_1,\ldots,p_r\}\subseteq\X$ be finite.
The trace of $\mathcal C$ on $P$ is
\[
\Pi_{\mathcal C}(P)
:=
\{(g(p_1),\ldots,g(p_r)) : g\in\mathcal C\}.
\]
Thus, $\Pi_{\mathcal C}(P)$ is the set of all dichotomies induced by
$\mathcal C$ on $P$.

The growth function of $\mathcal C$ is defined by
\[
\Pi_{\mathcal C}(r)
:=
\max_{P\subseteq\X:\,|P|=r}
|\Pi_{\mathcal C}(P)|.
\]

The VC dimension of $\mathcal C$, denoted $\VC(\mathcal C)$, is the
largest $r$ for which there exists a set $P\subseteq\X$ of size $r$
such that $|\Pi_{\mathcal C}(P)|=2^r$.
\end{definition}
\begin{lemma}[Sauer's lemma]\citep{sauer1972density}\label{Sauer's lemma}
We used Sauer's lemma in the form
\[
|\Pi_{\mathcal C}(P)|
\le
\left(\frac{er}{\VC(\mathcal C)}\right)^{\VC(\mathcal C)}
\]
whenever \(r>\VC(\mathcal C)\) 

\end{lemma}

\begin{theorem}[Small-error all-pairs generalization]\label{thm:bns-generalization}\citet[Theorem~B.5]{beimel2015learning}
	
Let $C$ and $\cD$ be a concept class and a distribution over a domain $X$.
Let $\alpha,\beta>0$, and
$
m \ge
\frac{48}{\alpha}
\left(
10\VC(C)\log\left(\frac{48e}{\alpha}\right)
+
\log\left(\frac{5}{\beta}\right)
\right).
$
Suppose that we draw a sample $S=(x_i)_{i=1}^m$, where each $x_i$ is drawn
i.i.d.\ from $\cD$. Then,
\[
\Pr\left[
\begin{array}{c}
\exists\, c,h\in C \text{ s.t. } \operatorname{disc}_{\cD_X}(c,h) \ge \alpha\\
\text{and } \operatorname{disc}_S(c,h)\le \alpha/10
\end{array}
\right]
\le \beta.
\]
Equivalently, with probability at least \(1-\beta\), simultaneously for every \(h,g\in\mathcal C\),
\[
\disc_S(h,g)\le\frac{\alpha}{10}
\quad\Longrightarrow\quad
\disc_{\cD_X}(h,g)<\alpha.
\]
\end{theorem}

\begin{theorem}[Realizable VC generalization bound]
\label{thm:realizable-vc-generalization}
Let $\mathcal C\subseteq\{0,1\}^{\mathcal X}$ be a binary hypothesis
class with $\operatorname{VC}(\mathcal C)\le v$. Fix an unlabeled samples multiset $S$ 
and let $B\subseteq S$
be a uniformly random subset of size $m$, and let $h^\star\in\cC$. For
$\tau,\beta\in(0,1)$,
There exists a universal constant \(c>0\) such that, if $|S| \ge m$ and
\[
m \ge c\frac{v\log(2/\tau)+\log(2/\beta)}{\tau}
\]
Then, with probability at least $1-\beta$, simultaneously for every $h\in\mathcal C$,
\[
\operatorname{disc}_{B}(h,h^\star)=0
\quad\Longrightarrow\quad
\operatorname{disc}_{S}(h,h^\star)\le\tau .
\]
\end{theorem}

\begin{proof}
The claim follows from Theorem~A2.2 of
\citet{blumer1989learnability}, applied to
the uniform distribution over $S$. Although their theorem is stated
for independent draws with replacement, the same guarantee holds for
a uniformly random subset of size $m$. Indeed, couple the independent
sample with a uniformly random $m$-subset $B$ containing all distinct
sampled points. Every hypothesis consistent with $B$ is also
consistent with the independent sample, and the result follows.
\end{proof}

\subsection{LabelBoost}
\label{subsec:labelboost}
We use \textsc{LabelBoostProcedure} of
\citet[Algorithm~2]{beimel2015learning}. Given a labeled sample \(S\), an
unlabeled sample \(T\), and a nonempty concept class \(\mathcal C\),
\textsc{LabelBoostProcedure} constructs one representative for every labeling
of \(U(S)\cup T\), then selects one representative hypothesis from $\mathcal{C}$ 
using the exponential mechanism with score equal to its agreement with \(S\), 
and relabels \(S\cup T\) according to the selected hypothesis.
We use \textsc{LabelBoostProcedure} with $D =\emptyset$ and score $q(S,g) = \sum_{j\in[|S|]}\mathbf{1}[g(x_j)=y_j]$

we will conduct extended privacy analysis of labelboost.

\begin{algorithm}
\caption{LabelBoostProcedure}
\begin{algorithmic}[1]
    \Statex \textbf{Input:} A partially labeled database
    $B = S\circ T \in (X\times\{0,1,\bot\})^*$,
    and a concept class $\mathcal{C}\subseteq\{0,1\}^X$.

    \State Initialize $H=\emptyset$.

    \State Let $P=\{p_1,\ldots,p_k\}$ be the set of all points
    $p\in X$ appearing at least once in $S\circ T$.

    \State For every
    $(z_1,\ldots,z_k)\in\Pi_{\mathcal{C}}(P)$,
    add to $H$ an arbitrary $g\in\mathcal{C}$
    such that $g(p_i)=z_i$ for every $1\le i\le k$.

    \State Choose $h\in H$ using the exponential mechanism
    with privacy parameter $\varepsilon=1$,
    solution set $H$, and database $S$.

    \State Relabel $S\circ T$ using $h$, and denote the resulting
    labeled database by $(S\circ T)^h$.

    \State \Return $(S\circ T)^h$.
\end{algorithmic}
\end{algorithm}

\begin{lemma}[LabelBoost privacy {\citep[Lemma~4.1]{beimel2015learning}}]
\label{lem:labelboost-privacy-blackbox}
Let \(A\) be an \((\varepsilon,\delta)\)-differentially private algorithm.
Then the composition obtained by applying \textsc{LabelBoostProcedure} and
subsequently running \(A\) on the relabeled sample is
\[
(\varepsilon+3,\,4e\delta)\text{-differentially private}.
\]
\end{lemma}
\begin{proposition}[Exponential-mechanism utility
{\citep[Proposition~A.1]{beimel2015learning}}]
\label{prop:exp-mech-utility}
Let \(F\) be a finite hypothesis class and let
\(\alpha:=\min_{f\in F}\operatorname{err}_S(f).
\)
If \(h\) is selected from \(F\) by the exponential mechanism with privacy
parameter \(\varepsilon\) and score equal to empirical agreement with \(S\),
then for every \(\lambda>0\),
\[
\Pr\!\left[
\operatorname{err}_S(h)>\alpha+\lambda
\right]
\le
|F|\exp\!\left(-\frac{\varepsilon\lambda |S|}{2}\right).
\]
\end{proposition}\begin{corollary}[Approximate utility of
\textsc{LabelBoostProcedure}]
\label{cor:labelboost-approximate-utility}
Let \(\mathcal C\) be a concept class with
\(\operatorname{VC}(\mathcal C)\le v\), and let \(|S| = |T| = m\). Suppose that there exists \(g\in\mathcal C\) such that
$
\operatorname{err}_S(g)\le\alpha.
$
Let \(h\) be the hypothesis selected internally by
\(\textsc{LabelBoostProcedure}(S,T,\mathcal C)\).
For a sufficiently large universal constant \(c\), with a constant $\varepsilon=1$ if
\[
m
\ge
\frac{cv}{\alpha}
\log\left(\frac{2}{\alpha\beta}\right),
\]
then, with probability at least \(1-\beta\),
\[
\operatorname{err}_S(h)\le 2\alpha.
\]  
\end{corollary}

\begin{proof}
Apply the proof of
\citet[Claim~4.2]{beimel2015learning}, replacing its zero-error
representative by the representative of \(g\), whose empirical error
is at most \(\alpha\). By
\Cref{prop:exp-mech-utility}, except with probability \(\beta\), the
selected hypothesis incurs at most an additional \(\alpha\) empirical
error. Hence
\(\operatorname{err}_S(h)\le2\alpha\).
The final bound on \(m\) implies the preceding displayed condition.
\end{proof}
\subsection{Extended Privacy Analysis}\label{sec:LabelBoost:ext}

\label{subsec:labelboost-privacy-structure} 
This section refines the analysis of \textsc{LabelBoostProcedure} in a way
that allows us, later on, to analyze algorithms that explicitly exploit
the realizability of the current sample.\citet{beimel2015learning} already proved that LabelBoost followed by $(\varepsilon,\delta)-DP$ algorithm is $(\varepsilon+3,4e\delta)-DP$. We want to offer a more precise analysis. Consider neighboring inputs:
\[
B_1=S_1\circ T_1
\qquad\text{and}\qquad
B_2=S_2\circ T_2
\]
Let \(P_i\) denote the set of distinct points appearing in \(B_i\), and let
\(K:=P_1\cap P_2\).
By \cref{def:growth-function}, \(\Pi_{\mathcal C}(K)\) is the set of all
labelings for the common points \(K\) induced by hypotheses in \(\mathcal C\).  Following \citet[Lemma~4.1]{beimel2015learning}, for every 
labeling \(z\in\Pi_{\mathcal C}(K)\), define the weight
\(w_i(z)\) as the total probability assigned by the exponential mechanism
to choose a representative hypothesis $h_i$ which induces labeling $z$ on $K$. Thus, \(w_1(z)\) and \(w_2(z)\) are the
probability masses corresponding to the same common labeling \(z\) in
the executions on \(B_1\) and \(B_2\).
The following weight comparison is established in the proof of
\citet[Lemma~4.1]{beimel2015learning}.

\begin{lemma}[LabelBoost weight comparison]
\label{lem:labelboost-weight-comparison}
For every \(z\in\Pi_{\mathcal C}(K)\),

$$
w_1(z)\le 4e\,w_2(z)
\qquad\text{and}\qquad
w_2(z)\le 4e\,w_1(z).
$$

\end{lemma}

For convenience, define
$
c_{\mathrm{LB}}:=\ln(4e).
$

notice that for a fixed \(z\in\Pi_{\mathcal C}(K)\) we can group the associated hypothesis by the labeling $z$ and define them as $h|_K =z$, by the definition of \textsc{LabelBoostProcedure} there are one or two hypothesis that are associated with $z$ and every labeling of any subset of $K$ by $h|_K =z$ an hypothesis associated with $z$ is the same.

\begin{lemma}[LabelBoost with a realizable-private continuation]
\label{lem:labelboost-realizable-continuation}
Fix a concept class \(\mathcal C\). Suppose \(A\) is
\((\varepsilon,\delta)\)-differentially private on $\cC$-realizable inputs (\Cref{def:realizable-dp}).
Define the mechanism \(\mathcal B\) as follows: On input \((S,T)\), run
\textsc{LabelBoostProcedure} with concept class \(\mathcal C\), let \(h\)
be the selected hypothesis, and output \(A(S^h \circ T^h)\)\. Then \(\mathcal B\) is
\[
(\varepsilon+c_{LB},\,4e\delta)
\text{-differentially private}
\]
on arbitrary neighboring inputs.
\end{lemma}

\begin{remark}
\label{remark:LB}
\Cref{lem:labelboost-realizable-continuation} also applies when the
continuation uses only \(T^h\). Indeed, define
\(\widetilde A(S^h,T^h):=A(T^h)\). Projecting neighboring pairs onto
their second component produces equal or neighboring samples, and
\(T^h\) remains realizable by \(h\in\mathcal C\). Hence
\(\widetilde A\) has the same realizable-input privacy guarantee.
\end{remark}

\begin{proof}
Consider neighboring inputs $B_1 =(S_1\circ T_1),B_2=(S_2\circ T_2)$, and denote $K$ as the unlabeled common set. 
For every \(z\in\Pi_{\mathcal C}(K)\), all representatives associated
with \(z\) in either execution belong to \(\mathcal C\). Hence every
relabeled input \(B_i^h\) passed to \(A\) is realizable by
\(\mathcal C\). Moreover, for each fixed \(z\), every
relabeled samples from the first execution is equal or neighboring to
every corresponding relabeled sample from the second execution, In particular assume that the set of distinct feature values appearing in \(B_1\) is
\(K\cup\{x\}\), and that appearing in \(B_2\) is \(K\cup\{x'\}\). for a fixed $z$ any labeling according to $h|_K =z$  has the same labels on $K$. \(B_1^{h_1}\) and \(B_2^{h_2}\) agree on every unchanged
record and are therefore equal or neighboring, for every
\(h_1,h_2\) associated with \(z\).
Consequently, for every event \(F\), and since $B_i^h$ is realizable by $h $ that associated with $z$, 
\begin{align*}
\Pr[\mathcal B(B_1)\in F]
&=
\sum_{z\in\Pi_{\mathcal C}(K)}
w_1(z)
\Pr\!\left[
A(B_1^h)\in F
\,\middle|\,
h|_K=z
\right]\\
&\le
4e
\sum_{z\in\Pi_{\mathcal C}(K)}
w_2(z)
\left(
e^\varepsilon
\Pr\!\left[
A(B_2^h)\in F
\,\middle|\,
h|_K=z
\right]
+\delta
\right)
\\
&=
4e\,e^\varepsilon\Pr[\mathcal B(B_2)\in F]+4e\delta
\\
&\le
e^{\varepsilon+c_{LB}}\Pr[\mathcal B(B_2)\in F]+4e\delta,
\end{align*}
where the first inequality holds by \Cref{lem:labelboost-weight-comparison} and by the assumption on $A$.
The reverse direction follows by interchanging \(B_1\) and \(B_2\).
\end{proof}

\begin{lemma}[Carrier-only LabelBoost release]
\label{lem:carrier-only-labelboost}
Fix a concept class $\mathcal{C}$ and an unlabeled carrier $T$.
Define the mechanism $M_T$ as follows: on input a labeled sample $S$,
compute
$$
T^h \leftarrow \textsc{LabelBoostProcedure}(S,T,\mathcal{C})
$$
where $h$ is the internally selected hypothesis and outputs the relabeled carrier $T^h$.
Then, for every fixed unlabeled carrier $T$, the mechanism $M_T$ is
$(c_{LB},0)$-differentially private with respect to its input sample $S$.
\end{lemma}
\begin{proof}
Let $T$ be a fix unlabeled carrier, let $S_1,S_2$ be a neighboring labeled samples. and let $h$ denote the selected hypothesis such that $M_T$ outputs $T^h$.
Let $B_1 = S_1 \circ T$ and $B_2 = S_2 \circ T$ 
, Apply \Cref{lem:labelboost-weight-comparison} to the neighboring
inputs $B_1,B_2$
Since the same carrier \(T\) appears in both inputs, every point appearing in
\(T\) belongs to the common set \(K\).

Fix \(z\in\Pi_{\mathcal C}(K)\). Every representative hypothesis associated with \(z\) induce exactly the same labeling on \(T\). Hence, for a fixed $z$, the labeling of $T$ is determined and fixed in both executions.

Let \(F\) be any event in the space of labeled versions of \(T\), let
\[
\mathcal Z_F
:=
\left\{
z\in\Pi_{\mathcal C}(K):
T^z\in F
\right\}.
\]
Let $T^{h_i}\xleftarrow[]{}M_T(S_i)$ then
\[
\Pr[T^{h_1}\in F]
=
\sum_{z\in\mathcal Z_F} w_1(z).
\]
By \Cref{lem:labelboost-weight-comparison},
\[
\Pr[T^{h_1}\in F]
\le
4e\sum_{z\in\mathcal Z_F} w_2(z)
=
4e\,\Pr[T^{h_2}\in F].
\]
Since \(4e=e^{c_{\mathrm{LB}}}\),
\[
\Pr[T^{h_1}\in F]
\le
e^{c_{\mathrm{LB}}}\Pr[T^{h_2}\in F].
\]
The reverse inequality follows symmetrically,
Hence \[
\Pr[M_T(S_1)\in F]
\le
e^{c_{\mathrm{LB}}}\Pr[M_T(S_2)\in F].
\]
Therefore the carrier-only
release $M_T$ is \((c_{\mathrm{LB}},0)\)-differentially private.
\end{proof}

\section{Component-by-Component Learning Framework}
\label{sec:component-framework}
In this section, we present a learner that constructs a hypothesis one component at a
time. To capture this process, we introduce an abstract component
oracle for each step. The oracle receives the components already chosen
by the learner together with the current labeled sample, and selects the
next component. Its utility guarantee requires that, with high
probability, the selected components can still be completed to a
hypothesis in the full concept class that closely approximates the
current sample. The following definitions formalize this framework.
\subsection{Component Representations}
\label{subsec:component-representation}

Let
\(\cC\subseteq\{0,1\}^{\cX}\)
be a concept class with a \(k\)-component parameterization.
Let \(\cY_{\le k}\) denote the set of valid parameter tuples, and
write \(c_{y_1,\ldots,y_k}\) for the concept represented by
\((y_1,\ldots,y_k)\). Thus,
\[
\cC
=
\left\{
c_{y_1,\ldots,y_k}:
(y_1,\ldots,y_k)\in\cY_{\le k}
\right\}.
\]

For \(i\in[k]\), write \(y_{\le i}:=(y_1,\ldots,y_i)\), and set
\(\cY_{\le0}:=\{\emptyset\}\). Define the set of valid prefixes of
length \(i\) by
\[
\cY_{\le i}
:=
\left\{
y_{\le i}:
\exists\,y_{i+1},\ldots,y_k
\text{ such that }
(y_1,\ldots,y_k)\in\cY_{\le k}
\right\}.
\]
Thus, \(y_{\le i}\in\cY_{\le i}\) if it can be extended
to a complete representation of a concept in \(\cC\).

For every valid prefix \(y_{\le i}\), define the restricted class
\[
\cC(y_{\le i})
:=
\left\{
c_{z_1,\ldots,z_k}:
(z_1,\ldots,z_k)\in\cY_{\le k}
\text{ and }
(z_1,\ldots,z_i)=y_{\le i}
\right\}.
\]
Hence, if
\(y_{\le i}=(y_{\le i-1},y_i)\in\cY_{\le i}\), then
\[
\cC(y_{\le i})
\subseteq
\cC(y_{\le i-1}).
\]
A hypothesis \(g\in\cC(y_{\le i})\) is called a
\emph{completion} of \(y_{\le i}\). In particular,
$
\cC(y_{\le k})
=
\{c_{y_1,\ldots,y_k}\}.
$
\begin{claim}
\label{claim:restricted-class-vc-dimension}
Suppose that $\VC(\cC)=v$. Then, for every $i\in[k]$ and every
$y_{\le i}\in\cY_{\le i}$,
\[
    \VC\bigl(\cC(y_{\le i})\bigr)\le v.
\]
\end{claim}

\begin{proof}
Since $\cC(y_{\le i})\subseteq\cC$, the claim follows from the
monotonicity of VC dimension under restriction to a subclass.
\end{proof}
\begin{definition}[Realizable prefix state]
\label{def:realizable-prefix-state}
A pair \((y_{\le i},S)\), where
\(y_{\le i}\in\cY_{\le i}\) and \(S\) is a labeled sample, is called a
\emph{realizable prefix state} if there exists
\(h_i\in\cC(y_{\le i})\) such that
\[
S=U(S)^{h_i}.
\]
Any such \(h_i\) is called a witness for \((y_{\le i},S)\). With this notations, $\operatorname{err}_S(h_i)=0$. \emph{realizable prefix state} $(y_{\le i},S)$ admits that $S$ is $\cC(y_{\le i})-$realizable
\end{definition}
\begin{definition}[Differential privacy on realizable prefix states]
\label{def:realizable-input-dp}
A prefix state \((y_{\le i},S)\) is realizable prefix state  if  $S$ is $\cC(y_{\le i})-$realizable

A randomized mechanism
\[
A:\mathcal Y_{\le i}\times\mathcal Z^m\to\mathcal R
\]
is \((\varepsilon,\delta)\)-differentially private on realizable
prefix states if, for every fixed \(y_{\le i}\in\mathcal Y_{\le i}\),
the mechanism \(A(y_{\le i};\cdot)\) is
\((\varepsilon,\delta)\)-differentially private on
\(\mathcal C(y_{\le i})\)-realizable inputs (see
\Cref{def:realizable-dp}).
consequently we say that if \((y_{\le i},S)\) and \((y_{\le i},S')\) are realizable prefix states and $S,S'$ are neighbors then \(A(y_{\le i},S)\) and \(A(y_{\le i},S')\) are $(\varepsilon,\delta) $ indistinguishable.
\end{definition}

This guarantee is weaker than standard differential privacy: it imposes
no requirement on neighboring \(S,S'\) if either
\((y_{\le i},S)\) or \((y_{\le i},S')\) is not a realizable prefix
state.

\subsection{Private Component Oracles}
\label{subsec:component-oracles}

Once several components have been fixed, the learner must choose the
next component without losing the ability to complete them to an
accurate hypothesis in the full class. We capture this task through a
private component oracle. Given the current prefix and labeled sample,
the oracle selects the next component while preserving, with high
probability, a good completion of the extended prefix. Since the
learner invokes these oracles on realizable prefix states, their privacy
and utility guarantees are formulated for such inputs. The precise
interface is given below.
\begin{assumption}[Private realizable component oracle]
\label{ass:component-oracle}

A sequence of $k$ algorithms 
\begin{align*}
	i = 1,\ldots, k: \quad A_i:
	\cY_{\le i-1}\times\cZ^m
	\longrightarrow
	\cY_i,
\end{align*}
with privacy parameters $\eps,\delta$ and utility parameters $\tau,\beta$, are called \emph{private component oracles} with sample complexities $M_{A,1}, \ldots, M_{A,k}$, if the following holds for any choice of $\delta, \tau, \beta \in (0,1)$ with $\delta \leq \beta$ and $\eps \in(0, 4]$: 

\begin{enumerate}
    \item \textbf{Totality and validity.}
    For every \(S\in\cZ^m\), the output
    $
    y_i\leftarrow A_i(y_{\le i-1};S)
    $
     is defined and satisfies
    $
    y_{\le i} = (y_{\leq i-1}, y_i)\in\cY_{\le i}.
    $

    \item \textbf{Privacy on realizable state.}
    The mechanism $
    A_i(y_{\le i-1};\cdot)
    $
    is \((\varepsilon,\delta)\)-differentially private on
    realizable inputs$-\cC(y_{\le i-1})$ (\Cref{def:realizable-input-dp}).
    \item \textbf{Utility} Let $(y_{\le i-1}, S)$ be a realizable state and witness by $h_{i-1}\in \cC(y_{\le i-1})$. Suppose
    $
    m\ge M_{A,i}(\tau,\beta,\varepsilon,\delta).
    $
    Then, with probability at least \(1-\beta\), for
    \(y_i\leftarrow A_i(y_{\le i-1};S)\) and
    \(y_{\le i}=(y_{\le i-1},y_i)\), there exists completion $
    g_i\in\cC(y_{\le i})
    $
    such that
    $
    \operatorname{disc}_{U(S)}(g_i,h_{i-1})\le\tau.$
\end{enumerate}
\end{assumption}

For a series $A_1,\ldots A_k$ we write
$
M_A(\tau,\beta,\varepsilon,\delta)
:=
\max_{i\in[k]} M_{A,i}(\tau,\beta,\varepsilon,\delta).
$

\section{Generic Learner}
\label{sec:generic-learner}
In this section, we present a generic algorithm \textsc{CarryForwardLearner} for
\(\mathcal C\), assuming access to component oracles \(A_1,\ldots,A_k\). The learner constructs the component vector
sequentially using the component oracle, then uses \textsc{LabelBoostProcedure} to carry a labeled sample from one round to the next while keeping the labeled samples realizable with respect to $\cC$. We are now ready to state the main guarantee of \textsc{CarryForwardLearner}, and in the next sections, we provide the privacy and utility analysis of the algorithm. 

\begin{theorem}[Carry-Forward Learner Guarantee] \label{thm:carry-forward-main}
Let \(\cC\subseteq\{0,1\}^{\cX}\)
be a concept class with a \(k\)-component parameterization and with $VC(\cC) = v$. Let $A_1,\ldots, A_k$ be an oracle sequence for the components in the sense of \Cref{ass:component-oracle}. For every $\varepsilon, \alpha,\beta,\delta \in (0,1)$ with $\delta \leq \beta$, there  exists a $(\varepsilon,\delta)$-differentially private $(\alpha,\beta,n)$-proper PAC-learner 
$A_{\textsc{CarryForwardLearner}}$ with sample complexity
\[
n
=
O\!\left(
\frac{\sqrt{k\log(2/\delta)}}{\varepsilon}
\max\left\{
\frac{kv}{\alpha}
\log\left(\frac{320k^2}{\alpha\beta}\right),
\;
M_A\!\left(
\frac{\alpha}{40k},
\frac{\beta}{4k},
4,
\frac{\delta}{16ek}
\right)
\right\}
\right).
\]
\end{theorem}
\begin{algorithm}[H]
\caption{\textsc{CarryForwardLearner}}
    \label{alg:carry-forward}
\begin{algorithmic}[1]
\Statex \textbf{Input:}
A labeled sample \(S\), target parameters
\(\alpha,\beta,\varepsilon,\delta\in(0,1)\), and component oracle sequence 
\(A_1,\ldots,A_k\) satisfies \cref{ass:component-oracle} 
\Statex \textbf{Parameters:}
Concept class \(\mathcal C\) represented by \(k\) components, and active sample size $m =  \left\lfloor
\frac{\varepsilon n}
{12e^9\sqrt{k\log(2/\delta)}}
\right\rfloor$
\State
Set
$
\tau\gets\frac{\alpha}{40k},
\:
\beta_A\gets\frac{\beta}{4k}, \: \varepsilon_A=4,\: \delta_A=\frac{\delta}{16ek}
$
\State Partition $S$, independently of record values, into
$
W^{\mathrm{init}},\: W^{(0)},\: W^{(1)},
$
each of size $\lfloor \frac{n}{3} \rfloor$.

\State Set $y_{\le 0}\gets\emptyset$.

\State Sample uniformly without replacement
$
R_0\subseteq W^{\mathrm{init}},
 |R_0|=m.
$

\State Run
$\textsc{LabelBoostProcedure}
\bigl(R_0,\emptyset,\mathcal C(y_{\le 0})\bigr)$, and let 
$h_0\in\mathcal C(y_{\le 0})=\mathcal C$ be the internally
selected hypothesis.  Set $S_0\gets U(R_0)^{h_0}$.

\For{$i=1,\ldots,k$}

    \State
    $
    y_i
    \leftarrow
    A_i
    \bigl(
        y_{\le i-1};
        S_{i-1}
    \bigr).
    $

    \State Set $
    y_{\le i}\gets(y_{\le i-1},y_i).
    $

    \If{$i=k$}
        \State \Return the hypothesis $c_{y_1,\ldots,y_k}$ represented by the full prefix
        $y_{\le k}$.
    \EndIf

    \State Sample uniformly without replacement
    \[
    T_i
    \subseteq
    W^{((i-1)\bmod 2)},
    \qquad
    |T_i|=m,
    \]

    \State Run $
    \textsc{LabelBoostProcedure}
    \bigl(
        S_{i-1},
        U(T_i),
        \mathcal C(y_{\le i})
    \bigr).
    $

    \State Let $h_i\in\mathcal C(y_{\le i})$ be the (internally) selected hypothesis
    and set $
    S_i\gets U(T_i)^{h_i}.
    $
\EndFor
\end{algorithmic}
\end{algorithm}
The key invariant is that every round, the 
state \((y_{\le i},S_i)\) remains realizable by
\(\mathcal C(y_{\le i})\). Therefore, every component oracle is
invoked within its promised realizable-input domain.
We first establish this invariant.

Note that immediately before
each invocation of \(A_i\), the current prefix and active sample form
a realizable prefix state. This property holds for every execution of
the learner and does not depend on the utility success of the oracles.

\begin{lemma}[Iterated realizability]
\label{lem:iterated-realizability}
At each iteration \(i\in[k]\), before the oracle call
$
y_i\leftarrow A_i(y_{\le i-1};S_{i-1}), 
$
there exists \(h_{i-1}\in\mathcal C(y_{\le i-1})\) such that 
$
S_{i-1}=U(S_{i-1})^{h_{i-1}}.
$
Consequently, \((y_{\le i-1},S_{i-1})\) is a realizable prefix state (\Cref{def:realizable-prefix-state}),
and every component oracle is invoked on a realizable input.

In particular for $i\ge1$: for every $(y_{\le i},T_{i})$ the state $(y_{\le i},S_i)$ is realizable prefix state. 
\end{lemma}

\begin{proof}
For \(i=1\), the preliminary LabelBoost call selects
\(h_0\in\mathcal C(y_{\le0})=\mathcal C\) and sets
$
S_0=U(R_0)^{h_0}.
$ 
Thus \((y_{\le0},S_0)\) is realizable.

Now fix \(i<k\). By the validity of \(A_i\), its output \(y_i\)
defines a valid prefix \(y_{\le i}\). The subsequent call to
\textsc{LabelBoostProcedure} selects
\(h_i\in\mathcal C(y_{\le i})\) and sets
$
S_i=U(T_i)^{h_i}.
$
Hence \((y_{\le i},S_i)\), which is the input state of \(A_{i+1}\),
is realizable. The claim follows by induction.
\end{proof}

The proof of \Cref{thm:carry-forward-main} is given in the following two sections: The privacy guarantee is proven in \Cref{sec:privacy} and the utility guarantee is proven in \Cref{sec:utility}.

\section{Privacy Analysis}
\label{sec:privacy}
We now analyze the privacy of \textsc{CarryForwardLearner}. A natural
approach would use one reservoir, amplify the privacy of each
oracle--LabelBoost iteration, and then apply adaptive composition.
This does not work directly because iteration \(i\) depends on the
reservoir through both the carried sample \(S_{i-1}\) and the newly
sampled carrier \(T_i\). Moreover, \(S_i\) cannot simply be included
in the private transcript: relabeling preserves its underlying
records, so neighboring inputs may be distinguished by the presence
of their differing record. Omitting \(S_i\), on the other hand, leaves
the state of the next iteration undetermined.

We therefore alternate between two reservoirs. When analyzing
\(W^{(b)}\), the incoming sample is carried by \(W^{(1-b)}\). A record
sampled from \(W^{(b)}\) may enter \(S_i\) and affect the next
component, but after the following relabeling the active sample is
again carried by \(W^{(1-b)}\). Hence a record from $W^{(b)}$ affect only the selected
components and the labeling of carrier from   \(W^{(1-b)}\). at the end of 2 iterations, We group these two iterations into a
\emph{macro-block}, which captures the complete influence of one
activation of \(W^{(b)}\).

Throughout this section, let $A_1,\ldots,A_k$ be a sequence of algorithms that satisfy \Cref{ass:component-oracle} with respect to sample complexity $m$ as defined in \Cref{alg:carry-forward}.

\subsection{The Two-Round Macro-Block}
\label{subsec:activated-macro-block}
We first analyze one reservoir activation at the carrier level. Fix
\(b\in\{0,1\}\) and an iteration \(i\) in which
\(T_i\subseteq W^{(b)}\), and fix the state
\(r_i=(y_{\le i-1},S_{i-1})\). For this fixed state, the macro-block
\(\operatorname{MB}^{r_i}_i\) is the mechanism that receives \(T_i\)
and performs the next two iterations of
\textsc{CarryForwardLearner}.

The following carrier \(T_{i+1}\) is sampled from the disjoint
reservoir \(W^{(1-b)}\). Hence a change in \(T_i\) can affect the
second iteration only through the relabeled sample \(S_i\), while the
underlying carrier \(T_{i+1}\) remains unchanged. Using the extended
LabelBoost analysis of \Cref{sec:LabelBoost:ext}, we prove privacy of
\(\operatorname{MB}^{r_i}_i\) with respect to neighboring carriers
\(T_i,T_i'\). Privacy with respect to the full reservoir \(W^{(b)}\)
is derived later by subsampling amplification.

\begin{definition}[Two round macro-block]
\label{def:macroblock}
Fix \(1\le i\le k-1\), and recall that \(T_i\) is a size-\(m\) carrier
drawn from a reservoir \(W^{((i-1) \bmod 2)}\), and therefore, independent of the reservoir $W^{(i \bmod 2)}$. Define the state
at the start of iteration \(i\) by
\[
r_i:=(y_{\le i-1},S_{i-1}, W^{(i \bmod 2)}).
\]
For a fixed state \(r_i\), define
\(\operatorname{MB}^{r_i}_{i}(T_i)\) by the following computation:
\begin{enumerate}
    \item Compute
    \[
    y_i\leftarrow A_i(y_{\le i-1},S_{i-1}),
    \qquad
    y_{\le i}:=(y_{\le i-1},y_i).
    \]

    \item Run
    \[
    h_i\leftarrow
    \textsc{LabelBoostProcedure}
    (S_{i-1},U(T_i),\mathcal C(y_{\le i})),
    \qquad
    S_i:=U(T_i)^{h_i}.
    \]

    \item Compute
    \[
    y_{i+1}\leftarrow A_{i+1}(y_{\le i},S_i),
    \qquad
    y_{\le i+1}:=(y_{\le i},y_{i+1}).
    \]

    \item If \(i\le k-2\), independently sample a fresh size-\(m\)
    carrier \(T_{i+1}\subseteq W^{(i \bmod 2)}\) uniformly without replacement,
    and set
    \[
    h_{i+1}\leftarrow
    \textsc{LabelBoostProcedure}
    (S_i,U(T_{i+1}),\mathcal C(y_{\le i+1})),
    \qquad
    S_{i+1}:=U(T_{i+1})^{h_{i+1}}.
    \]
\end{enumerate}

For the privacy analysis, define the augmented output by
\[
\operatorname{MB}^{r_i}_i(T_i)
:=
\begin{cases}
(y_i,y_{i+1},S_{i+1}), & 1\le i\le k-2,\\
(y_{k-1},y_k,\emptyset), & i=k-1.
\end{cases}
\]
The real learner does not reveal \(S_{i+1}\); it is included only as an
analytical device for representing the state passed to future activations.
\end{definition}
\begin{lemma}[Privacy of two-round macro-block]
\label{lemma:activated-two-round-macro-block}
For every fixed state $r_i:=(y_{\le i-1},S_{i-1}, W^{(i \bmod 2)})$ and for every \(1\le i\le k-1\), the mechanism $\operatorname{MB}^{r_i}_i$ (\Cref{def:macroblock}) over size-\(m\) carriers $T_i$ is \((C,\Delta)\)-differentially private, where
\[
C:=\varepsilon_A+2\cdot c_{\mathrm{LB}},
\qquad
\Delta:=4e\,\delta_A.
\]
\end{lemma}
\begin{proof}

The proof has two steps. First, we view the second round as a
mechanism whose input is the relabeled carrier
\(S_i=T_i^{h_i}\), and prove that this mechanism is private on
realizable neighboring inputs. We then apply
\Cref{lem:labelboost-realizable-continuation} to the preceding
LabelBoost call. Since LabelBoost passes a realizable relabeled
carrier to the private mechanism, the lemma gives privacy of the
entire macro-block with respect to \(T_i\).

\paragraph{Step 1: Analyzing the second round.}
For fixed \(u = (y_{\le i}, W^{(i \bmod 2)})\), we define \(K_u\) as a mechanism  of the second round of macro-block:

$K_u$, given an input \(S_i\) consists of labeled sample of size \(m\), acts as follows:

\begin{enumerate}
	\item Compute $y_{i+1}\xleftarrow{}A_{i+1}(y_{\leq i},S_i)$.
	\item If $i = k-1$, output $(y_{i+1},\emptyset)$.
	\item Choose uniformly an size-$m$ subset $T_{i+1} \subseteq W^{(i \bmod 2)}$.
	\item Execute $LabelBoostProcedure(S_i,U(T_{i+1}),\cC(y_{\leq i+1}))$ , and let  $h_{i+1}$ be the internally chosen hypothesis. 
	\item Output $(y_{i+1}, S_{i+1})$, for $S_{i+1} = U(T_{i+1})^{h_{i+1}}$.
\end{enumerate}

Note that steps 3-4 in $\MB_i^{r_i}$ can be replaced by calling $K_u(S_i)$ where $u = (y_{\leq i}, W^{(i \bmod 2)})$.

\begin{claim}\label{claim:second-round-privacy}
For a fixed $u$,
the mechanism \(K_u\) is 
$
(\varepsilon_A+c_{\mathrm{LB}},\delta_A)
\text{-differentially private}, 
$
on every pair of neighboring samples \(S_i,S_i'\) such that both
\((u,S_i)\) and \((u,S_i')\) are $\cC(y_{\leq i})$-realizable prefix states according to \Cref{def:realizable-prefix-state}
\end{claim}
\begin{proof}
For a fixed $u$, the class $\cC(y_{\le i})$ is also fixed. Hence, \Cref{ass:component-oracle} implies that
$A_{i+1}(u,S_i)$ and $A_{i+1}(u,S_i')$ are $(\eps_A,\delta_A)$-indistinguishable For $i = k-1$, the proof of the claim follows. 
We next assume that $i \leq k-2$. 

Fix an output \(y_{i+1}\) and a sample carried $T_{i+1} \subseteq \Wb$. Hence the   class
\(\mathcal C(y_{\le i+1})\) and the carrier \(U(T_{i+1})\) are fixed.
\Cref{lem:carrier-only-labelboost} therefore gives
\(c_{\mathrm{LB}}\)-differential privacy for the release of
\(S_{i+1} = U(T_{i+1})^{h_{i+1}}\). We conclude by (adaptive) basic composition that \(K_u\) is 
$
(\varepsilon_A+c_{\mathrm{LB}},\delta_A)
\text{-differentially private}
$
on $\cC(y_{\leq i})$-realizable prefix states.

\end{proof}

\paragraph{Step 2: Analyzing the entire process.}

Recall that $\MB_i^{r_i}(T_i)$, for $r_i = (y_{\leq i-1}, S_{i-1},\Wb)$, can be describe as follows:
\begin{enumerate}
	 \item Compute 
	$
	y_i\leftarrow A_i(y_{\le i-1},S_{i-1}).
	$
	
	\item Run
	$
	h_i\leftarrow
	\LabelBoost
	(S_{i-1},U(T_i),\mathcal C(y_{\le i})),
	$ and let $S_i:=U(T_i)^{h_i}$,
	
	\item Compute $(y_{i+1}, S_{i+1}) \gets K_u(S_i)$, where $u = (y_{\leq i}, \Wb)$.
	
	\item Output $(y_i, y_{i+1}, S_{i+1})$.
\end{enumerate}

To prove that $\MB_i^{r_i}$ is differentially private, fix neighboring $T_i, T_i'$ and consider the two executions $\MB_i^{r_i}(T_i)$ and $\MB_i^{r_i}(T_i')$.
Note that computing the $y_i$ part (step 1) is done independently of the input. Therefore, we assume without loss of generality that this part is fixed in both executions
The LabelBoost call is applied to the neighboring inputs
\[
(S_{i-1},U(T_i))
\sim
(S_{i-1},U(T_i')).
\]

Since \(K_u\) is applied to the relabeled carrier produced by
LabelBoost and is
\((\varepsilon_A+c_{\mathrm{LB}},\delta_A)\)-differentially private
on \(\mathcal C(y_{\le i})\)-realizable inputs, the case described in
\Cref{remark:LB} applies. Hence, by
\Cref{lem:labelboost-realizable-continuation}, the entire process is
\((\varepsilon_A+2\cdot c_{\mathrm{LB}},4e\delta_A)\)-differentially
private.

\end{proof}

\subsection{Reservoir privacy}
\label{subsec:reservoir-amplification}
We first establish the privacy of \textsc{CarryForwardLearner} with
respect to a change in a single record of \(W^{\mathrm{init}}\).
Since the initial reservoir is accessed only once, privacy with respect
to \(W^{\mathrm{init}}\) does not require any additional composition
analysis.
\begin{lemma}[Initialization privacy with respect to \(R_0\)]
\label{lem:initialization-active-sample-privacy}
Fix the reservoirs \(W^{(0)}\) and \(W^{(1)}\). Viewed as a mechanism of \(R_0\), the output of
\textsc{CarryForwardLearner} is
\[
(C,\Delta)
=
(\varepsilon_A+2\cdot c_{\mathrm{LB}},\,4e\delta_A)
\text{-differentially private}.
\]
\end{lemma}

\begin{proof}
Let \(K_{\emptyset}\) denote the continuation following the
preliminary LabelBoost call where On input \(S_0\), it can be described as follow:
\begin{enumerate}
    \item $y_1\leftarrow A_1(y_{\le0};S_0)$
    \item if k=1, output $(y_1,\emptyset)$
    \item Choose uniformly size-$m$ subset $T_1 \subseteq W^{(0)}$
    \item Execute $LabelBoostProcedure(S_0,U(T_{1}),\cC(y_{\leq 1}))$ , and let  $h_{1}$ be the internally chosen hypothesis. 
	\item Output $(y_{1}, S_{1})$, for $S_{1} = U(T_{1})^{h_{1}}$.

\end{enumerate}We first analyze \(K_{\emptyset}\) on neighboring
\(\mathcal C\)-realizable inputs \(S_0,S_0'\).
By \Cref{ass:component-oracle}, \(A_1\) is
\((\varepsilon_A,\delta_A)\)-differentially private on
these inputs. If \(k=1\), then \(K_{\emptyset}=A_1\).
Otherwise, after fixing the output \(y_1\) of \(A_1\)
and the independently sampled carrier \(T_1\), the
class \(\mathcal C(y_{\le 1})\) and the unlabeled
carrier \(U(T_1)\) are fixed. By
\Cref{lem:carrier-only-labelboost} and adaptive basic
composition, \(K_{\emptyset}\) is therefore
\((\varepsilon_A+c_{\mathrm{LB}},\delta_A)\)-differentially
private on neighboring \(\mathcal C\)-realizable
inputs in either case.

The preliminary LabelBoost call produces
\(S_0=U(R_0)^{h_0}\) for some \(h_0\in\mathcal C\).
Applying \Cref{lem:labelboost-realizable-continuation}
to this call followed by \(K_{\emptyset}\) shows that
the block output is
\((\varepsilon_A+2c_{\mathrm{LB}},4e\delta_A)\)-
differentially private with respect to \(R_0\).

The sample \(S_0\) is not released separately. After the initialization
block, the learner accesses \(R_0\) and \(S_0\) only through the block
output \((y_1,S_1)\), or through \(y_1\) when \(k=1\). Since the reservoirs are fixed in this analysis, all subsequent
computation is post-processing with respect to changes in \(R_0\).
Therefore the final output of \textsc{CarryForwardLearner} has the
same privacy guarantee.
\end{proof}

\begin{definition}[Amplified macro-block mechanism]
\label{def:amplified-macro-block}
Fix a round \(i\) and a state \(r_i\). Let $b={(i-1)\bmod 2}$
The mechanism \(\operatorname{AMB}_i^{r_i}\) takes a reservoir
\(W^{(b)}\) of size \(\lfloor n/3\rfloor\), samples
\(T_i\subseteq W^{(b)}\) uniformly without replacement with
\(|T_i|=m\), and outputs
\[
\operatorname{AMB}_i^{r_i}(W^{(b)})
:=
\operatorname{MB}_i^{r_i}(T_i).
\]
\end{definition}

\begin{corollary}[Reservoir amplification]
\label{cor:reservoir-amplification}
Let \(\mathcal M\) be a \((C,\Delta)\)-differentially private
mechanism on samples of size \(m\), and assume that \(\mathcal M\)
treats its input as a multiset of occurrences. Given a dataset \(W\) of size
\(\lfloor n/3\rfloor\), sample \(T\subseteq W\) uniformly without
replacement with \(|T|=m\), and output \(\mathcal M(T)\).

Then the resulting mechanism on \(W\) is
\((\eta,\zeta)\)-differentially private, where
\[
q:=\frac{m}{\lfloor n/3\rfloor},
\qquad
\eta:=\log\!\left(1+q(e^C-1)\right),
\qquad
\zeta:=q\Delta.
\]
Using \(\log(1+x)\le x\) and
\(\delta_A=\delta/(16ek)\), we have
\[
\eta
\le
\frac{m}{\lfloor n/3\rfloor}
\left(
e^{\varepsilon_A+2\cdot c_{\mathrm{LB}}}-1
\right),
\qquad
\zeta
=
\frac{m}{\lfloor n/3\rfloor}\frac{\delta}{4k}
\le
\frac{\delta}{4k}.
\]
\end{corollary}

\begin{proof}
This follows directly from the fixed-size subsampling amplification
guarantee of \Cref{thm:bbg}.
\end{proof}

\begin{remark}[Ordering of the sampled input]
\label{rem:carrier-ordering}
In \Cref{cor:reservoir-amplification}, we assume that
\(\mathcal M\) treats its input as a multiset. If this is not the
case, we uniformly randomly shuffle the sampled input before applying
\(\mathcal M\), as in \cite[Remark~4.4]{beimel2015learning}.
In \textsc{CarryForwardLearner}, this convention is applied to
\(R_0\) and each carrier \(T_i\) whenever the corresponding
initialization block or macro-block is not invariant to the ordering
of its input.
\end{remark}
\begin{corollary}[Initialization-reservoir privacy]
\label{cor:initialization-reservoir-privacy}
With respect to a change in one record of \(W^{\mathrm{init}}\), the
complete \textsc{CarryForwardLearner} is
\((\eta,\zeta)\)-differentially private, where
\[
q=\frac{m}{\lfloor n/3\rfloor},
\qquad
\eta=\log\!\left(1+q(e^C-1)\right),
\qquad
\zeta=q\Delta.
\]
\end{corollary}

\begin{proof}
By \Cref{lem:initialization-active-sample-privacy}, the complete
learner is \((C,\Delta)\)-differentially private as a mechanism of
\(R_0\). Since \(R_0\) is a uniform size-\(m\) subsample of
\(W^{\mathrm{init}}\), the claim follows directly from
\Cref{cor:reservoir-amplification}.
\end{proof}
\begin{corollary}[Privacy of the amplified macro-block]
\label{cor:amb-privacy}
For every fixed round \(i\) and state \(r_i\), the mechanism
\(\operatorname{AMB}_i^{r_i}\) is
\((\eta,\zeta)\)-differentially private, where
\[
q=\frac{m}{\lfloor n/3\rfloor},
\qquad
\eta=\log\!\left(1+q(e^C-1)\right),
\qquad
\zeta=q\Delta.
\]
\end{corollary}

\begin{proof}
By \Cref{lemma:activated-two-round-macro-block},
\(\operatorname{MB}_i^{r_i}\) is
\((C,\Delta)\)-differentially
private on carriers of size \(m\). The claim follows directly from
\Cref{cor:reservoir-amplification}.
\end{proof}

\subsection{Privacy Calibration}
\label{subsec:privacy-calibration}
In this section we want to conclude the final privacy of \textsc{CarryForwardLearner}, Previously, we established the entire algorithm privacy with respect to a change in
\(W^{\operatorname{init}}\). We now consider a change in
\(W^{(b)}, b\in\{0,1\}\).

Let $ b = (i-1) \bmod2$, fix $W^{init}, W^{(1-b)}$  in \Cref{subsec:reservoir-amplification} we established privacy of \(\operatorname{AMB}_i^{r_i}\) with respect to a change in $W^{b}$ and a fixed state $r_i =(y_{\le i-1},S_{i-1},W^{(1-b)})$. we now remove conditioning of $r_i$ using adaptive composition. We record \(\operatorname{AMB}_i^{r_i}\) output in a transcript and show that, conditioned on the preceding transcript, the
state of the next activation is fixed.

In order to record the transcript fix one reservoir \(W^{(b)}\), and enumerate its activations
by \(j\). let \(r_{j}\) denote $\operatorname{AMB}_j^{r_j}$ state, and let
$
O_j
$
be $\operatorname{AMB}_j^{r_j}$ augmented output.

Define the transcript through the \(j\)th activation by 
\[
\mathsf{tr}_{j}:=(O_1,\ldots,O_j),
\qquad
\mathsf{tr}_0:=\emptyset.
\]

\begin{lemma}[Transcript-determined boundary]
\label{lem:transcript-determined-boundary}
Fix \(b\in\{0,1\}\), and fix \(W^{\mathrm{init}}\) and
\(W^{(1-b)}\). Let \(i=2j-1+b\) be the round of the \(j\)th
activation of \(W^{(b)}\). Then the boundary state \(r_i\) is
determined by \((r_{1+b},\mathsf{tr}_{j-1})\).
Consequently, conditioned on
\((r_{1+b},\mathsf{tr}_{j-1})\), the release \(O_j\) is
\((\eta,\zeta)\)-differentially private with respect to \(W^{(b)}\).
\end{lemma}

\begin{proof}
The claim is immediate for \(j=1\). Let \(j\ge2\). Since consecutive
activations of \(W^{(b)}\) are separated by two rounds, the preceding
output is \(O_{j-1}=(y_{i-2},y_{i-1},S_{i-1})\). The preceding
transcript therefore determines \(y_{\le i-1}\) and \(S_{i-1}\).
Since \(W^{(1-b)}\) is fixed, it also determines the complete boundary
state \(r_i\).

Thus, conditioned on \((r_{1+b},\mathsf{tr}_{j-1})\), the boundary
state \(r_i\) is fixed. By \Cref{cor:amb-privacy}, the conditional
mechanism mapping \(W^{(b)}\) to \(O_j\) is
\((\eta,\zeta)\)-differentially private.
\end{proof}
\begin{corollary}[Privacy with respect to a reservoir]
\label{cor:adaptive-reservoir-composition}
Fix \(b\in\{0,1\}\), and consider neighboring inputs that differ in
one record of \(W^{(b)}\), while \(W^{\mathrm{init}}\) and
\(W^{(1-b)}\) are unchanged. Then
\textsc{CarryForwardLearner} is
\((\varepsilon_{\mathrm{res}},\delta_{\mathrm{res}})\)-differentially
private with respect to this change, where
\[
\ell:=\left\lceil\frac{k-1}{2}\right\rceil,
\]
and, for every \(\delta_{\mathrm{ac}}>0\),
\[
\varepsilon_{\mathrm{res}}
\le
\sqrt{2\ell\log(1/\delta_{\mathrm{ac}})}\,\eta
+
\ell\eta(e^\eta-1),
\qquad
\delta_{\mathrm{res}}
\le
\ell\zeta+\delta_{\mathrm{ac}}.
\]
\end{corollary}

\begin{proof}
The reservoir \(W^{(b)}\) supplies at most \(\ell\) carriers. By
\Cref{lem:transcript-determined-boundary}, conditioned on the initial
boundary \(r_{1+b}\) and each preceding transcript, every activation
is \((\eta,\zeta)\)-differentially private with respect to
\(W^{(b)}\). Adaptive composition
(\Cref{thm:drv}) gives the stated bounds for the joint transcript,
and hence for the output of \textsc{CarryForwardLearner} by
post-processing, conditioned on \(r_{1+b}\).

Before round \(1+b\), the algorithm does not access \(W^{(b)}\).
Since the other reservoirs are identical, coupling the two executions
with the same randomness gives the same initial boundary \(r_{1+b}\).
Averaging over its common distribution removes the conditioning while
preserving the same privacy parameters.
\end{proof}

\begin{lemma}[Privacy of \textsc{CarryForwardLearner}]
\label{lem:carry-forward-privacy}
Let \(\varepsilon\le1\) and \(\delta\in(0,1)\). Under the parameter
choices of \Cref{alg:carry-forward},
\(
\varepsilon_A=4,
\,
\delta_A=\frac{\delta}{16ek},
\,
m=
\left\lfloor
\frac{\varepsilon n}
{12e^9\sqrt{k\log(2/\delta)}}
\right\rfloor,
\)
\textsc{CarryForwardLearner} is
\((\varepsilon,\delta)\)-differentially private on arbitrary input
samples.
\end{lemma}

\begin{proof}
Let
\[
q:=\frac{m}{\lfloor n/3\rfloor},
\qquad
\eta:=\log\!\left(1+q(e^C-1)\right),
\qquad
\zeta:=4eq\delta_A.
\]
Using \(\log(1+x)\le x\), with $\varepsilon_A=4,\delta_A=\frac{\delta}{16ek}$ and $C=2c_{LB} + \varepsilon_A <9$ we obtain
\[
\eta
\le
q(e^C-1)
\le
\frac{\varepsilon}
{4\sqrt{k\log(2/\delta)}},
\qquad
\zeta
=
\frac{q\delta}{4k}
\le
\frac{\delta}{4k}.
\]

Set \(\delta_{\mathrm{ac}}:=\delta/2\). For \(k\ge2\), let
\[
\ell:=\left\lceil\frac{k-1}{2}\right\rceil\le k.
\]
we have \(\eta\le1\), and hence
\(e^\eta-1\le2\eta\). Therefore,
\[
\begin{aligned}
\varepsilon_{\mathrm{res}}
&\le
\sqrt{2\ell\log(2/\delta)}\,\eta
+
\ell\eta(e^\eta-1)\\
&\le
\frac{\sqrt{2}}{4}\varepsilon
+
\frac{2\varepsilon^2}
{16\log(2/\delta)}
\le
\varepsilon,
\end{aligned}
\]
and
\[
\delta_{\mathrm{res}}
\le
\ell\zeta+\delta_{\mathrm{ac}}
\le
\frac{\delta}{4}+\frac{\delta}{2}
<
\delta.
\]
Thus, by \Cref{cor:adaptive-reservoir-composition}, the complete
learner is \((\varepsilon,\delta)\)-differentially private with respect to a change in either reservoir. By
\Cref{cor:initialization-reservoir-privacy}, the same conclusion holds
with respect to a change in \(W^{\mathrm{init}}\). When \(k=1\), only
the latter case is required.

Finally, condition on the value-independent random partition of the
input. For every fixed partition, neighboring inputs differ in one
record of exactly one reservoir, while the other reservoirs are
identical. Therefore either
\Cref{cor:initialization-reservoir-privacy} or
\Cref{cor:adaptive-reservoir-composition} applies. Since the same
\((\varepsilon,\delta)\)-privacy guarantee holds for every fixed
partition, averaging over the partition randomness preserves it.
\end{proof}

\section{Utility Analysis}\label{sec:utility}
We now analyze the utility of \textsc{CarryForwardLearner}. Throughout this section, let $A_1,\ldots,A_k$ be a sequence of algorithms that satisfy \Cref{ass:component-oracle} with maximal sample complexity $M_A$.

At the
beginning of iteration \(i\), the active sample \(S_{i-1}\) is labeled
by a witness \(h_{i-1}\in\mathcal C(y_{\le i-1})\). The component
oracle selects \(y_i\) while ensuring that the new restricted class
\(\mathcal C(y_{\le i})\) contains a completion close to
\(h_{i-1}\) on \(S_{i-1}\). The subsequent
\textsc{LabelBoostProcedure} selects the hypothesis \(g_i\) used to
label the next carrier, and
\Cref{cor:labelboost-approximate-utility} ensures that \(g_i\) remains
empirically close to \(h_{i-1}\).

All-pairs generalization guarantee then converts this empirical
bound into population disagreement \(O(\tau)\). Thus every iteration
changes the carried hypothesis by at most \(O(\tau)\) in population
disagreement. After accounting for initialization and telescoping over
the \(k\) iterations, choosing \(\tau=\Theta(\alpha/k)\) gives the
desired error \(\alpha\). The following lemma establishes the
one-round guarantee used in this argument.
\begin{lemma}[One-round utility guarantee]
\label{lemma:one-round-utility}
Let $\cC$ be a concept class with $\VC(\cC)=v$, Let $\beta_{VC},\beta_{A},\beta_{LB} \in (0,1)$, and fix $i<k$.
Suppose that $(y_{\le i-1},S_{i-1})$ is a realizable prefix state,
witnessed by some
$h_{i-1}\in\cC(y_{\le i-1})$, and that
$U(S_{i-1})$ consists of $m$ i.i.d.\ samples from $\cD_X$.

$\textsc{LabelBoostProcedure}
(S_{i-1},U(T_i),\cC(y_{\le i}))$.

Let
$y_i\leftarrow A_i(y_{\le i-1};S_{i-1})$, and let
$g_i\in\cC(y_{\le i})$ be the hypothesis selected by

$\textsc{LabelBoostProcedure}
(S_{i-1},U(T_i),\cC(y_{\le i}))$.
For a sufficiently large
universal constant $c$, it suffices that
\[
m
\ge
\max\left\{
M_A(\tau,\beta_A,\varepsilon_A,\delta_A),
\frac{c}{\tau}
\left(
v\log\frac{2}{\tau\beta_{\mathrm{LB}}}
+
\log\frac{1}{\beta_{\mathrm{VC}}}
\right)
\right\}.
\]

Then, with probability at least
$1-\beta_A-\beta_{\mathrm{LB}}$,
\[
\operatorname{err}_{S_{i-1}}(g_i)
=
\operatorname{disc}_{U(S_{i-1})}(g_i,h_{i-1})
\le 2\tau.
\]
Moreover, with probability at least
$1-\beta_A-\beta_{\mathrm{LB}}-\beta_{\mathrm{VC}}$,
\[
\operatorname{disc}_{\cD_X}(g_i,h_{i-1})
<20\tau.
\]
\end{lemma}

\begin{proof}By the assumed choice of \(m\), the component-oracle, LabelBoost, and
all-pairs generalization sample requirements are all satisfied.

By \Cref{ass:component-oracle}, the component-oracle utility
guarantee implies that, except with probability \(\beta_A\), there
exists a completion \(h_i\in\mathcal C(y_{\le i})\) satisfying
\(\operatorname{disc}_{U(S_{i-1})}(h_i,h_{i-1})\le\tau\).
Since \(S_{i-1}\) is labeled by \(h_{i-1}\), this is equivalent to
\(\operatorname{err}_{S_{i-1}}(h_i)\le\tau\).

By \Cref{cor:labelboost-approximate-utility}, the LabelBoost utility
guarantee implies that, except with an additional failure probability
of at most \(\beta_{\mathrm{LB}}\),
\[
\operatorname{err}_{S_{i-1}}(g_i)
=
\operatorname{disc}_{U(S_{i-1})}(g_i,h_{i-1})
\le2\tau.
\]

Finally, by \Cref{thm:bns-generalization}, the all-pairs
generalization guarantee holds simultaneously for every pair of
hypotheses in \(\mathcal C\). Hence, except with probability
\(\beta_{\mathrm{VC}}\), the empirical bound above implies
\(\operatorname{disc}_{\mathcal D_X}(g_i,h_{i-1})<20\tau\).
A union bound gives total failure probability at most
\(\beta_A+\beta_{\mathrm{LB}}+\beta_{\mathrm{VC}}\).
\end{proof}

\begin{theorem}[Iterated utility]
\label{thm:iterated-utility}
Let $\cD$ be realizable by some $h^\star\in\cC$, where
$\VC(\cC)=v$, and suppose that each component oracle
$A_i$ satisfies \cref{ass:component-oracle}.
Set
\[
\tau:=\frac{\alpha}{40k},
\qquad
\beta_A,\beta_{\mathrm{LB}},\beta_{\mathrm{VC}}
\le \frac{\beta}{4k}.
\]
It suffices to take
\[
m
=
O\!\left(
\max\left\{
\frac{kv}{\alpha}
\log\!\left(\frac{320k^2}{\alpha\beta}\right),
\;
M_A\!\left(
\frac{\alpha}{40k},
\frac{\beta}{4k},
4,
\frac{\delta}{16ek}
\right)
\right\}
\right).
\]

Then, with probability at least $1-\beta$,
\textsc{CarryForwardLearner} outputs
\[
h_k:=c_{y_1,\ldots,y_k}\in\cC
\]
satisfying
\[
\operatorname{err}_{\cD}(h_k)\le\alpha.
\]
\end{theorem}
\begin{proof}Consider one round of \textsc{CarryForwardLearner} 
for \(i\in[k]\) .let $h_{i-1}\in \cC(y_{\le i-1})$ be the realizable state witness, let $g_i\xleftarrow[]{}\textsc{LabelBoostProcedure}
(S_{i-1},U(T_i),\mathcal C(y_{\le i}))$ the hypothesis completion. For $0<i<k$ define
\[
E_i
:=
\left\{
\operatorname{disc}_{\mathcal D_X}(g_i,h_{i-1})<20\tau
\right\}.
\]
By Lemma~\ref{lem:iterated-realizability}, every component oracle satisfies \Cref{ass:component-oracle} is
invoked on a realizable prefix state, and hence its utility guarantee
applies. Moreover, \(U(S_0)=U(R_0)\), while
\(U(S_{i-1})=U(T_{i-1})\) for \(i\ge2\). Each carrier is a uniform
selection of \(m\) records from an i.i.d.\ reservoir, Consequently,
$U(S_{i-1})$ consist $m$ i.i.d samples from $\cD_X$.
Hence \Cref{lemma:one-round-utility} guarantee applies in every round.

For initialization, define
$
E_0
:=
\left\{
\operatorname{disc}_{\mathcal D_X}(h_0,h^\star)<20\tau
\right\}.
$ Since \(h^\star\) has zero empirical error on \(R_0\) and $R_0$ is i.i.d samples from $\cD_X$, by the preliminary
LabelBoost guarantee and \Cref{thm:bns-generalization} imply
\(\Pr[E_0]\ge1-(\beta_{\mathrm{LB}}+\beta_{\mathrm{VC}})\).

For the last round, define
$
E_k
:=
\left\{
\operatorname{disc}_{\mathcal D_X}(h_{k-1},h_k)<20\tau
\right\}.
$
by oracle utility guarantee and \Cref{thm:bns-generalization} $\operatorname{Pr}[E_k] =1-\beta_A -\beta_{VC}$.

A union bound over \(E_0,E_1,\ldots,E_k\) now gives
\[
\Pr\left[\bigcap_{i=0}^k E_i\right]
\ge
1-
\left(
k\beta_A+k\beta_{\mathrm{LB}}+(k+1)\beta_{\mathrm{VC}}
\right)
\ge1-\beta.
\]
On \(\bigcap_{i=0}^k E_i\), the triangle inequality gives
\[
\begin{aligned}
\operatorname{err}_{\mathcal D}(h_k)
&=
\operatorname{disc}_{\mathcal D_X}(h_k,h^\star)\\
&\le
\operatorname{disc}_{\mathcal D_X}(h_0,h^\star)
+
\operatorname{disc}_{\mathcal D_X}(h_{k-1},h_{k})
+
\sum_{i=1}^{k-1}
\operatorname{disc}_{\mathcal D_X}(g_i,h_{i-1})\\
&<
20(k+1)\tau
\le\alpha.
\end{aligned}
\]
Hence by probability $1-\beta$ \textsc{CarryForwardLearner} outputs $c_{y_1\ldots y_k}$ and $\operatorname{err}_{\cD_X}(h_k) \le \alpha$
\end{proof}

\begin{corollary}[\textsc{CarryForwardLearner} utility] Let $\cC$ be a concept class with $k$-component parameterization and
$\VC(\cC)=v$. Let $\cD$ be a distribution realizable by some
$h^\star\in\cC$, and let $A_1,\ldots,A_k$ be the input oracle sequence with maximal sample complexity $M_A$.
Let $S\sim\cD^n$ be an i.i.d.\ sample of size $n$ given as input to
\textsc{CarryForwardLearner}. Then, for every $\varepsilon,\alpha,\beta,\delta\in(0,1)$.
where 
\[
n
=
O\!\left(
\frac{\sqrt{k\log(2/\delta)}}{\varepsilon}
\max\left\{
\frac{kv}{\alpha}
\log\left(\frac{320k^2}{\alpha\beta}\right),
\;
M_A\!\left(
\frac{\alpha}{40k},
\frac{\beta}{4k},
4,
\frac{\delta}{16ek}
\right)
\right\}
\right).
\]
Then, with probability at least $1-\beta$,
\textsc{CarryForwardLearner} outputs
$c_{y_1,\ldots,y_k}\in\cC$ satisfying
$\operatorname{err}_{\cD}(c_{y_1,\ldots,y_k})\le\alpha$.
\end{corollary}
\begin{proof}
by \Cref{alg:carry-forward}\[
n
=
O\!\left(
\frac{m\sqrt{k\log(2/\delta)}}{\varepsilon}
\right).
\]
by \cref{thm:iterated-utility} 
\[
m
=
O\!\left(
\max\left\{
\frac{kv}{\alpha}
\log\!\left(\frac{320k^2}{\alpha\beta}\right),
\;
M_A\!\left(
\frac{\alpha}{40k},
\frac{\beta}{4k},
4,
\frac{\delta}{16ek}
\right)
\right\}
\right).
\]
\end{proof}
\section{Applications}
\label{sec:applications}
In this section, we present applications of the generic
\textsc{CarryForwardLearner} framework. Each application is obtained by
constructing a sequence of component oracles satisfying
\Cref{ass:component-oracle}. We give three applications: the first two
concern private learning of halfspaces, while the third concerns closure
under Boolean composition.

\subsection{Common Halfspace Coordinate Framework}
\label{subsec:halfspace-coordinate-framework}
We develop a common coordinate framework for the two component-oracle
instantiations below: the private quasi-concave optimizer of
\citet{cohen2022qc} and the \textsc{IPConcave} algorithm of
\citet{nissim2026quasiconcave}.
Both applications optimize a quasi-concave function for the linear feasibility
problem.
\paragraph{The Linear Feasibility Problem}
In this problem, we are given a finite grid
\(\mathcal G=[[\pm X]]:=\{z\in\mathbb Z:|z|\le X\}\)
for some \(X\in\mathbb N\). A dataset
$
S^{\mathrm{LP}}\in(\mathcal G^d\times\mathcal G)^*,
$
such each \((x,y)\in S^{\mathrm{LP}}\) represents the linear
constraint
\(\langle x,a\rangle\ge y\)
which defined by the halfspace $h_{x,y}\in \mathbb{R}^d$ 
on a point \(a\in\mathbb R^d\).
The depth of \(a\) with respect to \(S^{\mathrm{LP}}\) is
$
\operatorname{depth}_{S^{\mathrm{LP}}}(a)
:=
\left|
\left\{
(x,y)\in S^{\mathrm{LP}}:
\langle x,a\rangle\ge y
\right\}
\right|,
$
namely, the number of constraints in \(S^{\mathrm{LP}}\) satisfied by
\(a\). The linear-feasibility problem is to find a point
\(a\in\mathbb R^d\) of large depth.

For the halfspace application, we identify the labels
\(\{0,1\}\) with \(\{-1,1\}\). Thus,
$
S=((x_j,\ell_j))_{j=1}^n
\in(\mathcal X^d\times\{-1,1\})^n,
$
where \(\mathcal X^d\subseteq\mathbb R^d\) is finite, the oracle consider the standard reduction from halfspace to linear-feasibility in order to optimize the depth function over the halfspace parameters, each coordinate oracle output represent a matching halfspace coordinate.

A key observation is that optimizing the $\operatorname{depth}$ function one coordinate at a time in over a realizable dataset result with an quasi-concave function. hence we are able to to use \textsc{CarryForwardLearner} to preserve realizable prefix state and optimize $\operatorname{depth}$ one coordinate at a time using quasi-concave optimization. 

In \Cref{app:halfspace-linear-feasibility}, we prove that the
reduction from halfspaces to linear feasibility preserves
realizability:
For the linear-feasibility to halfspace reduction, we assume
general position (defined by \citet{kaplan2020halfspaces}):
a set of points in \(\mathbb R^d\) is in general position if no
\(d+1\) of them lie on the same hyperplane.
Under this assumption, the transfer from constraint violations to
classification errors incurs at most \(d\) additional errors.

Since general position need not hold initially, we follow \citet{kaplan2020halfspaces} and apply a sufficiently small perturbation that preserves
realizability of the initial labeled sample and places its points
in general position with high probability.
In the coordinate construction and statements below, \(S\) denotes
an active sample on the fixed perturbed and rescaled grid described
in \Cref{app:halfspace-linear-feasibility}.

To simplify notation, we write
\(\operatorname{depth}_S\) for the depth with respect to
\(S^{\mathrm{LP}}\).
Thus \(\operatorname{depth}_S(p)\) counts the constraints induced
by \(S\) that are satisfied by a full parameter point
\(p=(a,w)\in\mathbb R^{d+1}\).

Write the halfspace parameter vector in the order
\[
y=(w,a_1,\ldots,a_d),
\qquad k=d+1.
\]

\begin{definition}[Coordinate depth objective]
\label{def:depth-coordinate-function}
Fix a prefix \(y_{\le i-1}\). For a labeled sample \(S\), define
\[
Q_i^S(z)
:=
\max_{y_{i+1},\ldots,y_k\in\mathbb R}
\operatorname{depth}_S
\bigl(
y_{\le i-1},z,y_{i+1},\ldots,y_k
\bigr).
\]
Thus \(Q_i^S(z)\) is the largest number of constraints induced by \(S\)
that can still be satisfied after fixing the \(i\)th coordinate to
\(z\). Changing one example in \(S\) changes one induced constraint,
so \(Q_i^S\) has sensitivity one.
\end{definition}

Both \citet{kaplan2020halfspaces} and
\citet{nissim2026quasiconcave} use the objective
\(\operatorname{cdepth}\) in order to apply a quasi-concave
optimization algorithm. We use the iterated realizability property
of \textsc{CarryForwardLearner}, which ensures that every
\((y_{\le i-1},S_{i-1})\) is a realizable prefix state
(see \Cref{lem:iterated-realizability}).
The following lemma shows that, under prefix realizability, the
coordinate objective \(Q_i^S\), defined directly from
\(\operatorname{depth}\), is quasi-concave on its finite domain.

\begin{definition}[Finite coordinate domains.]
\label{def:finite-coordinate-domains}
We set \(\widetilde{\mathcal Y}_1:=\{-1,1\}\).
For \(i=2,\ldots,k\), after fixing
\(y_{\le i-1}=(w,a_1,\ldots,a_{i-2})\), with each preceding
coordinate chosen from its corresponding finite domain, let
\(\widetilde{\mathcal Y}_i(y_{\le i-1})\) be the finite domain for
\(a_{i-1}\) obtained from the construction of
\citet[Definition~4.4]{kaplan2020halfspaces}, described in
\cref{app:halfspace-linear-feasibility}, applied to the
\(d\)-dimensional linear-feasibility instance obtained after fixing
\(w\). The domains are constructed for the perturbed and rescaled
finite grid as in
\citet[proof of Theorem~B.7]{kaplan2020halfspaces}  \footnote{We
use the corrected perturbation resolution from
\Cref{lem:perturbation-general-position}, implemented in
\Cref{lem:perturbation-preserves-realizability}.
This enlarges the grid and the resulting finite-domain size
bounds, but adds only \(O(1)\) to the bound on \(\log^*N\).
Consequently, the asymptotic sample complexity of our
applications is unchanged. (see \Cref{par:finite-domains-after-perturbation})}

Throughout both oracle instantiations, let
\[
N
:=
\max_{\substack{
i\in[k],\,
y_{\le i-1}\in\mathcal Y_{\le i-1}\\
y_j\in\widetilde{\mathcal Y}_j(y_{\le j-1})
\ \text{for all }j<i}}
\left|
\widetilde{\mathcal Y}_i(y_{\le i-1})
\right|.
\]
\end{definition}

\begin{lemma}[Realizable completion in finite domains]
\label{lem:normalized-finite-domains-statement}
Let \(i\in[k]\), and let \((y_{\le i-1},S)\) be a realizable
prefix state with \(m=|S|\), where the prefix is generated by
the domains of \Cref{def:finite-coordinate-domains}.
Then \(Q_i^S\) is sensitivity-one and quasi-concave on
\(\widetilde{\mathcal Y}_i(y_{\le i-1})\), and there exists
\(y_i\in\widetilde{\mathcal Y}_i(y_{\le i-1})\)
such that \(y_{\le i}\) admits a completion satisfying all \(m\)
linear constraints induced by \(S\). Equivalently,
\[
Q_i^S(y_i)
=
\max_{z\in\widetilde{\mathcal Y}_i(y_{\le i-1})}Q_i^S(z)
=
m.
\]
\end{lemma}

\begin{proof}
See \Cref{lem:normalized-finite-domains}.
\end{proof}

\begin{corollary}[Depth to disagreement transfer]
\label{cor:depth-to-disagreement}
Let \((y_{\le i-1},S)\) be a realizable prefix state, witnessed by
\(h_{i-1}\in\mathcal C(y_{\le i-1})\), where the prefix is
generated by the domains of
\Cref{def:finite-coordinate-domains}.
Suppose that \(U(S)\) is in general position.
If \(y_i\in\widetilde{\mathcal Y}_i(y_{\le i-1})\) admits a
completion \(g_i\in\mathcal C(y_{\le i})\) satisfying at least
\(m-E\) of the linear constraints induced by \(S\), where
\(m=|S|\), then
\[
\operatorname{disc}_{U(S)}(g_i,h_{i-1})
\le \frac{E+d}{m}.
\]
\end{corollary}

\begin{proof}
Apply \Cref{cor:lp-to-halfspace-appandix} to the parameter vector
representing \(g_i\). It makes at most \(E+d\) errors on \(S\).
Since \(S=U(S)^{h_{i-1}}\), these errors are precisely its
disagreements with \(h_{i-1}\) on \(U(S)\).
\end{proof}

\begin{theorem}[Halfspace learning from component oracles under general position]
\label{thm:halfspace-framework}
Let \(\mathcal C\) be the class of halfspaces over the finite
grid \(\mathcal X^d\subseteq\mathbb R^d\), with the coordinate
representation above and \(k=d+1\).
Suppose that, for every prefix generated by
\Cref{def:finite-coordinate-domains}, the component oracles satisfy
the requirements of \Cref{ass:component-oracle}, except that utility
is required only on realizable prefix states in general position.
On every labeled input, each oracle outputs an element of its
corresponding finite coordinate domain.
Let \(M_A\) be their common sample requirement for the perturbed
and rescaled grid.

For every \(\alpha,\beta,\varepsilon,\delta\in(0,1)\) with
\(\delta\le\beta\), there exists an
\((\varepsilon,\delta)\)-differentially private proper
\((\alpha,\beta)\)-PAC learner for \(\mathcal C\) in the realizable
setting, with sample complexity
\[
n=
O\!\left(
\frac{\sqrt{k\log(4/\delta)}}{\varepsilon}
\max\left\{
\frac{kd}{\alpha}
\log\!\left(\frac{1280k^2}{\alpha\beta}\right),
\;
M_A\!\left(
\frac{\alpha}{80k},
\frac{\beta}{8k},
4,
\frac{\delta}{32ek}
\right)
\right\}
\right).
\]
\end{theorem}

\begin{proof}
See \Cref{proof:under-genereal-position}
\end{proof}

\subsection{A Coordinate Oracle via Private Quasi-Concave Optimization}
\label{subsec:cohen-coordinate-oracle}
In this application we use the quasi-concave optimizer
of Cohen et al.\ Section~4.1.

\citet{cohen2022qc} reduces private quasi-concave optimization
to a cumulatively private interior-point problem,
\citet[Theorem~4.2]{cohen2022qc} provides the required interior-point algorithm. we defined the algorithm as \textsc{QCOpt}. Together with \Cref{subsec:halfspace-coordinate-framework} the function $Q_i^S$ (see \Cref{def:depth-coordinate-function}) is quasi concave with sensitivity 1, we apply \textsc{QCOpt} with $Q_i^S$ and samples $S$ on finite domain in order to yield the following black-box guarantee.

\begin{proposition}[Cohen et al., Section~4.1 and the proof of Theorem~4.2]
\citep{cohen2022qc}
\label{prop:cohen-qc}
Let \(\mathcal Y\) be a finite totally ordered domain of size \(N\), and
let \(Q(S,\cdot):\cY \xrightarrow[]{} \mathbb{R}\) be a sensitivity-one
quasi-concave quality function. Then there exists an
\((\varepsilon_A,\delta_A)\)-differentially private algorithm
\textsc{QCOpt} such that, with probability at least
\(1-\delta_A\), it outputs \(\widehat y\in\mathcal Y\) satisfying
\[
Q(S,\widehat y)
\ge
\max_{y\in\mathcal Y}Q(S,y)-E_{\mathrm{QC}},
\]
where
\[
E_{\mathrm{QC}}
=
\widetilde O\!\left(
\frac{2^{\log^\ast N}\log(1/\delta_A)}{\varepsilon_A}
\right).
\]
\end{proposition}

Here \(\widetilde O\) suppresses polynomial factors in
\(\log^\ast N\).

In the following, we denote by $\mathsf{QCOpt}\left(
S,Q,
\cY,
\varepsilon,\delta
\right)$ the algorithm $\mathsf{QCOpt}$ executed on a dataset $S$ over a finite domain $\cY$, and optimizes the function $Q$ using privacy parameters $\eps,\delta$.
\paragraph{Coordinate oracle.}
By \Cref{lem:normalized-finite-domains-statement}, \(Q_i^S\) satisfies
the sensitivity and quasi-concavity requirements of
\Cref{prop:cohen-qc} on every realizable prefix state. On input
\((y_{\le i-1};S)\), the oracle constructs \(Q_i^S\) and
\(\widetilde{\mathcal Y}_i(y_{\le i-1})\) as above and returns
\[
A_i^{\mathrm{QC}}(y_{\le i-1};S)
:=
\mathsf{QCOpt}_{\varepsilon_A,\delta_A}
\left(
S,Q_i^S,\widetilde{\mathcal Y}_i(y_{\le i-1})
\right).
\]

\paragraph{Total extension.}
For every fixed prefix, the domain
\(\widetilde{\mathcal Y}_i(y_{\le i-1})\) is finite and nonempty.
We use the following total extension of \(\mathsf{QCOpt}\).
The mechanism executes the construction of Cohen et al.; if the
construction fails to produce a valid input for one of its
subroutines, or if its output does not belong to
\(\widetilde{\mathcal Y}_i(y_{\le i-1})\), it returns
\[
\min\widetilde{\mathcal Y}_i(y_{\le i-1})
\]
instead.

\begin{lemma}[Realizable coordinate oracle]
\label{lemma:cohen-coordinate-oracle}The coordinate oracle 
\(A_i\) satisfies \Cref{ass:component-oracle} under general position and the prefix generated by \Cref{def:finite-coordinate-domains} with failure 
parameter \(\beta_A \ge\delta_A\). Hence, with probability at
least \(1-\beta_A\), its output \(y_i\) that admits a completion
\(g_i\in\mathcal C(y_{\le i})\) satisfying
\[
\operatorname{disc}_{U(S)}(g_i,h_{i-1})
\le
\frac{E_{\mathrm{QC}}+d}{m}.
\]
Hence, if
\[
m\ge\frac{E_{\mathrm{QC}}+d}{\tau},
\]
then \(\operatorname{disc}_{U(S)}(g_i,h_{i-1})\le\tau\).
Consequently, recall $E_{\mathrm{QC}}
=
\widetilde O\!\left(
\frac{2^{\log^\ast N}\log(1/\delta_A)}{\varepsilon_A}
\right)$ hence:
\[
M_A(\tau,\delta_A,\varepsilon_A,\delta_A)
=
\widetilde O\!\left(
\frac{
2^{\log^\ast N}\log(1/\delta_A)/\varepsilon_A+d
}{\tau}
\right).
\]
\end{lemma}

\begin{proof} Fix a prefix generated by
\Cref{def:finite-coordinate-domains}.
By the total extension, the oracle is always defined and returns
\(y_i\in\widetilde{\mathcal Y}_i(y_{\le i-1})\).
The extended prefix is therefore generated by the same construction
and, since \(w\in\{-1,1\}\), admits a halfspace completion by setting
the remaining coefficients to zero.
This proves totality and validity.
when $(y_{\le i-1},S)$ is realizable prefix the totality follow from $A_i$ definition otherwise totality guarantee from the Total extension we provided.    

For privacy and utility we let \((y_{\le i-1},S)\) be a realizable prefix state, witnessed by
\(h_{i-1}\in\mathcal C(y_{\le i-1})\), and let \(m=|S|\).

For privacy analysis fix \(y_{\le i-1}\), and let \(S,S'\) be neighboring samples such
that both \((y_{\le i-1},S)\) and \((y_{\le i-1},S')\) are realizable
prefix states. By
\Cref{lem:normalized-finite-domains-statement}, the objectives
\(Q_i^S\) and \(Q_i^{S'}\) are sensitivity-one and quasi-concave.
\Cref{prop:cohen-qc} therefore implies that
\(A_i(y_{\le i-1};\cdot)\) is
\((\varepsilon_A,\delta_A)\)-differentially private on realizable input (see \Cref{def:realizable-input-dp}) and satisfying the privacy requirement of  
\Cref{ass:component-oracle}.

Moreover \Cref{lem:normalized-finite-domains-statement} gives
$
\max_{z\in\widetilde{\mathcal Y}_i(y_{\le i-1})}
Q_i^S(z)=m.
$
Hence, with probability at least \(1-\beta_A\),
\Cref{prop:cohen-qc} gives
$
Q_i^S(y_i)
\ge
m-E_{\mathrm{QC}}.
$

Assume 
witness \(h_{i-1}\in\cC(y_{\le i-1})\) for $S$ such $S=U(S)^{h_{i-1}}$.
by the definition of \(Q_i^S\), the prefix \(y_{\le i}\) admits a
completion satisfying at least \(m-E_{\mathrm{QC}}\) constraints
induced by \(S\). under general position \Cref{cor:depth-to-disagreement} implies that exist a
corresponding halfspace
\(g_i\in\mathcal C(y_{\le i})\) that disagrees on at most \(E_{\mathrm{QC}}+d\) sample points. 
Therefore
\[
\operatorname{disc}_{U(S)}(g_i,h_{i-1})
\le
\frac{E_{\mathrm{QC}}+d}{m}.
\]
\end{proof}

\begin{theorem}[Private realizable learning of halfspaces with quasi-concave optimization]
\label{thm:private-realizable-halfspaces}
Let \(\mathcal C\) be the class of \(d\)-dimensional halfspaces over
the finite grid \(\mathcal X^d\).
For every \(\alpha,\beta,\varepsilon,\delta\in(0,1)\), $\beta \ge \delta $there exists
an \((\varepsilon,\delta)\)-differentially private proper
\((\alpha,\beta)\)-PAC learner for \(\mathcal C\) in the realizable
setting with sample complexity
\[
\widetilde O\!\left(
\frac{d^{5/2}+d^{3/2}2^{\log^\ast N}}
{\varepsilon\alpha}
\right).
\]
\end{theorem}

\begin{proof}
Assume $\beta_A \ge \delta_A$, with the parameterization of \Cref{thm:halfspace-framework}: $\delta_A = \frac{\delta}{32ek} \le \frac{\delta}{8k} \le \beta_A$ with
\(\varepsilon_A=4\) by
\Cref{lemma:cohen-coordinate-oracle} 
the oracle $A_i$ satisfy \Cref{ass:component-oracle} under general position and prefix generated by finite domains as in\Cref{def:finite-coordinate-domains}, Hence by \Cref{thm:halfspace-framework}
\(
M_A(\tau,\beta_A,4,\delta_A)
=
\widetilde O\!\left(
\frac{2^{\log^\ast N}+d}{\tau}
\right).
\) Hence
substituting this bound into \Cref{thm:halfspace-framework},
with \(k=d+1\) and \(\tau=\Theta(\alpha/d)\), proves the claim.
\end{proof}

\subsection{A Coordinate Oracle via \textsc{IPConcave}}
\label{subsec:ipconcave-coordinate-oracle}

We instantiate the component oracle using the \textsc{IPConcave}
algorithm of \citet{nissim2026quasiconcave}, with the coordinate objective $\operatorname{depth}$
and finite domains established in
\Cref{subsec:halfspace-coordinate-framework}.
\begin{theorem}[\citet{cohen2022qc}]
Let \(\mathcal X\) be a finite ordered domain. There exists an
\((\varepsilon,\delta)\)-differentially private algorithm
\textsc{PrivateIP} that on input \(S\in\mathcal X^n\) outputs an
interior point with probability \(1-\beta\), provided that
$
n>
n_{\mathrm{IP}}(|\mathcal X|,\beta,\varepsilon,\delta),
$ 
for
$ 
n_{\mathrm{IP}}(|\mathcal X|,\beta,\varepsilon,\delta)
\in
O\!\left(
\frac{
\log^\ast|\mathcal X|
\cdot
\log^2\!\left(
\frac{\log^\ast|\mathcal X|}{\beta\delta}
\right)
}{\varepsilon}
\right).
$ 
\end{theorem}
\begin{algorithm}[H]
\caption{\textsc{IPConcave} \cite{nissim2026quasiconcave}}
\label{alg:ipconcave}
\begin{algorithmic}[1]
\Require Dataset $S\in\cX^n$, finite domain $\widetilde{\cX}\subseteq\mathbb{R}$,
score function $Q:\cX^*\times\widetilde{\cX}\to\mathbb{R}$, and parameters
$\alpha,\beta,\varepsilon,\delta,t$.
\State Randomly partition $S$ into $B_1,\ldots,B_t$, each of size at least $\lfloor n/t\rfloor$.
\For{$i\in[t]$}
    \State $y_i\gets \arg\max_{x\in\widetilde{\cX}} Q(B_i,x)$.
\EndFor
\State \Return
$\hat{x}\gets
\textsc{PrivateIP}_{\widetilde{\cX},\beta,\varepsilon,\delta}
(y_1,\ldots,y_t)$.
\end{algorithmic}
\end{algorithm}

\begin{theorem}
[\citet{nissim2026quasiconcave}]
\label{thm:ipconcave}
For every $\varepsilon>0$ and $\alpha,\beta,\delta\in(0,1)$,
\textsc{IPConcave} is $(\varepsilon,\delta)$-differentially private.
\end{theorem}
\paragraph{Coordinate oracle.}
On input \((y_{\le i-1};S)\), the oracle constructs \(Q_i\), where $Q_i$ is a max function over a finite domain for $\operatorname{depth}$, and
\(\widetilde{\mathcal Y}_i(y_{\le i-1})\) as above, compute $t \ge n_{IP}(|\Tilde{\cY_i}(y_{\le i-1)}|,\beta_A/2,\varepsilon_A,\delta_A) $ and returns
\[
A_i^{\mathrm{IP}}(y_{\le i-1};S)
:=
\textsc{IPConcave}_{\beta_A/2,\varepsilon_A,\delta_A,t}
\left(
S,Q_i,\widetilde{\mathcal Y}_i(y_{\le i-1})
\right).
\]
The following lemma replaces the generic
\(1/\tau^2\) approximation bound used for \textsc{IPConcave} \citet{nissim2026quasiconcave} by a realizable VC bound with
\(1/\tau\) dependence.

\begin{lemma}[Realizable transfer for linear feasibility]
\label{lemma:realizable-lp-transfer-statement}
Let \(S=((x_j,y_j))_{j=1}^n\) be a realizable labeled sample, and let
\(B\) be a uniformly random submultiset of \(m\) occurrences from
\(S\). Suppose that
\[
m
\ge
\frac{c}{\rho}
\left(
(d+1)\log\frac{2}{\rho}
+
\log\frac{1}{\beta}
\right)
\]
for a sufficiently large universal constant \(c\). Then, with
probability at least \(1-\beta\), every parameter point
\(p=(a,w)\in\mathbb R^{d+1}\) satisfying
\(\operatorname{depth}_B(p)=m\) also satisfies
$
\operatorname{depth}_S(p)\ge(1-\rho)n.
$
\end{lemma}
\begin{proof}
    see \ref{cor:realizable-lp-transfer}
\end{proof}

\begin{lemma}[Realizable transfer for block maximizers]
\label{lem:ipconcave-realizable-transfer}
Let \((y_{\le i-1},S)\) be realizable prefix state, witnessed by
\(h_{i-1}\in\mathcal C(y_{\le i-1})\), and let \(m=|S|\). Suppose
that every block \(B_j\) satisfies
\[
|B_j|
\ge
\frac{c}{\rho}
\left(
d\log\frac{2}{\rho}
+
\log\frac{t}{\beta_{\mathrm{net}}}
\right)
\]
for a sufficiently large universal constant \(c\). Then, with
probability at least \(1-\beta_{\mathrm{net}}\), every block maximizer
\(z_j\) satisfies
\[
Q_i^S(z_j)\ge(1-\rho)m.
\]
\end{lemma}

\begin{proof}
Fix a block \(B_j\). Since \(B_j\subseteq S\), the state
\((y_{\le i-1},B_j)\) is realizable prefix state. By
\Cref{lem:normalized-finite-domains-statement},
\[
\max_{z\in\widetilde{\mathcal Y}_i(y_{\le i-1})}
Q_i^{B_j}(z)
=
|B_j|.
\]
Hence every block maximizer \(z_j\) admits a completion satisfying all
constraints induced by \(B_j\).

Since \(B_j\) is a uniformly random
submultiset of \(S\), its induced constraints form a uniformly random
submultiset of the constraints induced by \(S\).
\Cref{lemma:realizable-lp-transfer-statement}, applied with failure
probability \(\beta_{\mathrm{net}}/t\), therefore implies that the same
completion satisfies at least \((1-\rho)m\) constraints induced by
\(S\). By the definition of \(Q_i^S\),
\[
Q_i^S(z_j)\ge(1-\rho)m.
\]
A union bound over the \(t\) blocks proves that this holds
simultaneously for every \(j\in[t]\), except with probability
\(\beta_{\mathrm{net}}\).
\end{proof}

\begin{lemma}[Coordinate oracle via \textsc{IPConcave}]
\label{lemma:ipconcave-coordinate-oracle}
For every \(i\in[k]\), the algorithm \(A_i^{\mathrm{IP}}\) satisfies
\Cref{ass:component-oracle} under general position and prefix generated by \Cref{def:finite-coordinate-domains} with sample requirement 
\[
|S|
\ge
M_A(\tau,\beta_A,\varepsilon_A,\delta_A)
:=
\widetilde O\!\left(
\frac{d}{\tau}
+
\frac{d\log^\ast N}{\varepsilon_A\tau}
\right).
\]
\end{lemma}

\begin{proof}
For every fixed prefix,
\citet[Theorem~9]{nissim2026quasiconcave} implies that
\(A_i^{\mathrm{IP}}\) is
\((\varepsilon_A,\delta_A)\)-differentially private on arbitrary
labeled samples, and hence also on realizable neighboring samples.
Its output belongs to
\(\widetilde{\mathcal Y}_i(y_{\le i-1})\), the extended prefix is generated by the finite-domain
construction and is valid, setting the remaining coefficient to zero admits a valid halfspace. totality follow from definition.

For utility, suppose that \((y_{\le i-1},S)\) is a realizable prefix
state and that \(U(S)\) is in general position. Choose
\(\rho=\tau/2\) and set
\(\beta_{\mathrm{IP}}=\beta_{\mathrm{net}}=\beta_A/2\).
Let \(y_i\) be the output of \textsc{IPConcave}.

For every \(j\in[t]\), let \(z_j\) be a maximizer of
\(Q_i^{B_j}\). By \Cref{lem:ipconcave-realizable-transfer}, except
with probability \(\beta_{\mathrm{net}}\), simultaneously for every
\(j\in[t]\),
\(
Q_i^S(z_j)\ge(1-\rho)m.
\)
Except with probability \(\beta_{\mathrm{IP}}\),
\textsc{PrivateIP} outputs
\(y_i\in[\min_j z_j,\max_j z_j]\). By
\Cref{lem:normalized-finite-domains-statement}, \(Q_i^S\) is
quasi-concave, and therefore
\[
Q_i^S(y_i)
\ge
\min_{j\in[t]}Q_i^S(z_j)
\ge
(1-\rho)m.
\]
A union bound shows that this holds with probability at least
\(1-\beta_{\mathrm{net}}-\beta_{\mathrm{IP}}=1-\beta_A\).

Let \(h_{i-1}\in\mathcal C(y_{\le i-1})\) be a witness realizing
\(S\). By the definition of \(Q_i^S\), the prefix \(y_{\le i}\)
admits a completion \(g_i\in\mathcal C(y_{\le i})\) satisfying at
least \((1-\rho)m\) constraints induced by \(S\). Since \(U(S)\) is
in general position, \Cref{cor:depth-to-disagreement} gives
\(
\operatorname{disc}_{U(S)}(g_i,h_{i-1})
\le \rho+d/m.
\)
Since \(\rho=\tau/2\), requiring \(m\ge2d/\tau\) gives
\(
\operatorname{disc}_{U(S)}(g_i,h_{i-1})\le\tau.
\)

Finally, \textsc{PrivateIP} uses
\(t=\widetilde O(\log^\ast N/\varepsilon_A)\) block maximizers, while
\Cref{lem:ipconcave-realizable-transfer} requires
\(
r=|B_j|=\widetilde O(d/\rho)
=\widetilde O(d/\tau)
\)
examples per block. Ensuring \(m\ge tr\), together with
\(m\ge2d/\tau\), gives
\[
M_A(\tau,\beta_A,\varepsilon_A,\delta_A)
=
\widetilde O\!\left(
\frac{d}{\tau}
+
\frac{d\log^\ast N}{\varepsilon_A\tau}
\right).
\]
\end{proof}
\begin{theorem}[Private realizable learning of halfspaces via
\textsc{IPConcave}]
\label{thm:private-realizable-halfspaces-ipconcave}
Let \(\mathcal C\) be the class of \(d\)-dimensional halfspaces over
the finite grid \(\mathcal X^d\).
For every \(\alpha,\beta,\varepsilon,\delta\in(0,1)\), $\beta \ge \delta$there exists
an \((\varepsilon,\delta)\)-differentially private proper
\((\alpha,\beta)\)-PAC learner for \(\mathcal C\) in the realizable
setting with sample complexity
\[
\widetilde O\!\left(
\frac{d^{5/2}\log^\ast N}{\varepsilon\alpha}
\right).
\]
\end{theorem}

\begin{proof}
By
\Cref{lemma:ipconcave-coordinate-oracle}, and  $\delta_A = \frac{\delta}{32ek} $ with \(\varepsilon_A=4\), the oracle $A_i$ satisfy \Cref{ass:component-oracle} under general position and prefix generated by finite domains as in \Cref{def:finite-coordinate-domains} with sample complexity
\(
M_A(\tau,\beta_A,4,\delta_A)
=
\widetilde O\!\left(
\frac{d\log^\ast N}{\tau}
\right).
\)
Applying and substituting this bound into \Cref{thm:halfspace-framework},
with \(k=d+1\) and \(\tau=\Theta(\alpha/d)\), proves the claim.
\end{proof}

\subsection{Closure under Boolean Composition}
\label{subsec:boolean-closure} 
In the preceding applications, each component oracle directly selects
a component while preserving a nearby completion in the current
restricted global class. In other settings, however, the available
primitive may only be a private learner \(B_i\) for a local class
\(\mathcal H_i\). A sample realizable by
\(\mathcal C(y_{\le i-1})\) need not be realizable by
\(\mathcal H_i\), so the realizable utility guarantee of \(B_i\)
cannot be applied directly.

We use the relabeling construction of
\citet[Algorithm~2]{alon2020closure} to bridge this gap.
The construction selects a local hypothesis that admits a good global
completion, relabels the sample by this hypothesis, and then applies
\(B_i\). Properness ensures that the learned hypothesis remains a
valid local component, while the structural conditions below lift its
local utility guarantee to a nearby hypothesis in the global class.

The following conversion was proved by
\citet[Lemma~5.9]{bun2015thresholds} and is restated as
Lemma~18 by \citet{alon2020closure}.

\begin{lemma}[PAC-to-empirical conversion]
\label{lem:pac-to-empirical}
Let \(B\) be an
\((\varepsilon,\delta)\)-differentially private,
\((\alpha,\beta)\)-accurate PAC learner for a concept class
\(\mathcal H\), with sample complexity \(m\).
Given a realizable labeled sample of size \(9m\), let
\(\widehat B\) sample \(m\) examples uniformly with replacement and
apply \(B\) to the resulting sample. Then \(\widehat B\) is an
\((\varepsilon,\delta)\)-differentially private, 
\((\alpha,\beta)\)-accurate empirical learner for \(\mathcal H\).
Moreover, if \(B\) is proper, then so is \(\widehat B\).
\end{lemma}

\begin{proposition}[Component-oracle construction]
\label{prop:local-component-oracle-construction-statement}
Fix \(i\in[k]\), a prefix \(y_{\le i-1}\), and a labeled sample
\(S=((x_j,y_j))_{j=1}^m\), Let $B_i$ be the proper local $(\varepsilon_B,\delta_B)$-differentially private $(\tau,\beta_B)-accurate$ PAC learner for $\cH_i$. where $1\ge \varepsilon_B >0, \delta_B\in(0,1)$
For \(h\in\mathcal H_i\), define the
completion score
\begin{equation}
\label{eq:blackbox-score}
q_i(S,h):=
\min_{g\in\mathcal C(y_{\le i-1},h)}
\sum_{j=1}^m\mathbf 1[g(x_j)\neq y_j],
\end{equation}
Let \(\widehat B_i\) be the proper empirical learner obtained from
\(B_i\) by \Cref{lem:pac-to-empirical}. Define \(A_i\) by
\[
(\widetilde S,\emptyset)
\leftarrow
\textsc{ARelabel}(S,\emptyset;\mathcal H_i,q_i),
\qquad
y_i\leftarrow\widehat B_i(\widetilde S),
\]
and return \(y_i\)
\end{proposition}

\begin{corollary}[Local proper learners give component oracles]
\label{cor:blackbox-local-oracle-statment}
Suppose that every \(\mathcal H_i\) admits a proper
\((\varepsilon_B,\delta_B)\)-differentially private realizable PAC
learner \(B_i\), with sample complexity
\(m_{B,i}(\alpha,\beta,\varepsilon_B,\delta_B)\),  $1\ge \varepsilon_B >0, \delta_B\in(0,1)$. Assume that, for
every \(i\in[k]\) and every valid prefix \(y_{\le i-1}\), the following
conditions hold:
\begin{enumerate}
    \item Every \(h\in\mathcal H_i\) is a valid extendable component:
    \((y_{\le i-1},h)\in\mathcal Y_{\le i}\) and
    \(\mathcal C(y_{\le i-1},h)\neq\emptyset\).

    \item The local components cover the current restricted class:
    \(\mathcal C(y_{\le i-1})
    =\bigcup_{h\in\mathcal H_i}\mathcal C(y_{\le i-1},h)\).

    \item For every \(h,h'\in\mathcal H_i\) and every
    \(\bar g\in\mathcal C(y_{\le i-1},h)\), there exists
    \(g'\in\mathcal C(y_{\le i-1},h')\) such that
    \(\mathbf 1[g'(x)\neq\bar g(x)]
    \le\mathbf 1[h'(x)\neq h(x)]\) for every \(x\in\mathcal X\).
\end{enumerate}
Then \Cref{prop:local-component-oracle-construction-statement} yields component oracles
\(A_1,\ldots,A_k\) satisfying \Cref{ass:component-oracle} with
\(\varepsilon_A=\varepsilon_B+3\) and
\(\delta_A=4e\delta_B\).

Writing \(d_i=\operatorname{VC}(\mathcal H_i)\), it suffices to take
\begin{equation*}
M_{A,i}(\tau,\beta_A,\varepsilon_A,\delta_A)
=
\widetilde O\!\left(
\max\left\{
\frac{d_i+\log(1/\beta_A)}{\tau},
m_{B,i}\left(
\frac{\tau}{2},
\frac{\beta_A}{2},
\varepsilon_B,
\delta_B
\right)
\right\}
\right).
\end{equation*}
A common sample requirement is
\(M_A=\max_{i\in[k]}M_{A,i}\).
\end{corollary}

\begin{proof}
See Appendix~\ref{app:local-proper-blackbox}.
\end{proof}

\begin{lemma}[Closure oracles for \textsc{CarryForwardLearner}]
\label{lem:closure-oracle}
Let \(G:\{0,1\}^k\to\{0,1\}\), let
\(\mathcal H_1,\ldots,\mathcal H_k\subseteq\{0,1\}^{\mathcal X}\)
be nonempty classes, and define
$
\mathcal C=G(\mathcal H_1,\ldots,\mathcal H_k).
$

Suppose that every \(\mathcal H_i\) admits a proper
\((\varepsilon_B,\delta_B)\)-differentially private PAC learner
\(B_i\) with sample complexity
\(M_i(\alpha,\beta,\varepsilon_B,\delta_B)\).

Then the
construction of \Cref{prop:local-component-oracle-construction-statement}
provides component oracles \(A_1,\ldots,A_k\) satisfying
\Cref{ass:component-oracle}, with the privacy parameters and sample
requirements stated in
\Cref{cor:blackbox-local-oracle-statment}.
\end{lemma}

\begin{proof}
    
Fix a prefix \(h_{\le i-1}=(h_1,\ldots,h_{i-1})\).
For \(h\in\mathcal H_i\), define
\[
\mathcal C(h_{\le i-1},h)
=
\left\{
G(h_1,\ldots,h_{i-1},h,c_{i+1},\ldots,c_k):
c_j\in\mathcal H_j
\right\}.
\]
Hence the structural conditions of
\Cref{cor:blackbox-local-oracle-statment} hold immediately.
Indeed, every \(h\in\mathcal H_i\) is a valid extendable component,
$
\mathcal C(h_{\le i-1})
=
\bigcup_{h\in\mathcal H_i}
\mathcal C(h_{\le i-1},h),
$
and local replacement is obtained by keeping the same suffix
\(c_{i+1},\ldots,c_k\) and replacing only \(h\) by \(h'\).
If \(h(x)=h'(x)\), the inputs to \(G\) are identical at \(x\), and hence
the two resulting global hypotheses agree there.
\end{proof}
\begin{theorem}[Private realizable learning of Boolean compositions]
\label{thm:private-realizable-boolean}
Let \(G:\{0,1\}^k\to\{0,1\}\), let
\(\mathcal H_1,\ldots,\mathcal H_k\) be nonempty classes, and define
\(\mathcal C=G(\mathcal H_1,\ldots,\mathcal H_k)\).
Assume that every
\(\mathcal H_i\) admits a proper private realizable PAC learner
\(B_i\) with sample complexity
\(m_i(\alpha,\beta,\varepsilon_B,\delta_B)\). Let
\(d_i:=\operatorname{VC}(\mathcal H_i)\).

For every \(\alpha,\beta,\varepsilon,\delta\in(0,1)\), $\beta \ge \delta$ there exists an
\((\varepsilon,\delta)\)-differentially private
\((\alpha,\beta)\)-PAC learner for \(\mathcal C\) in the realizable
setting with sample complexity
\[
\widetilde O\left(
\frac{\sqrt{k}}{\varepsilon}
\max\left\{
\frac{k}{\alpha}
\max\left\{
\operatorname{VC}(\mathcal C),
\max_{i\in[k]}d_i
\right\},
\max_{i\in[k]}
m_i\left(
\frac{\alpha}{80k},
\frac{\beta}{8k},
1,
\frac{\delta}{64e^2k}
\right)
\right\}
\right).
\]
\end{theorem}
\begin{proof}
Choose
\(\varepsilon_B=1\),
\(\delta_B=\delta/(64e^2k)\),
\(\tau=\alpha/(40k)\), and
\(\beta_A=\beta/(4k)\).
By \Cref{lem:closure-oracle} and
\Cref{cor:blackbox-local-oracle-statment}, the resulting component
oracles have
\(\varepsilon_A=4\) which is sufficient for \Cref{thm:carry-forward-main}, \(\delta_A=\delta/(16ek)\), and common sample
requirement
\[
M_A
=
\widetilde O\left(
\max\left\{
\frac{k}{\alpha}\max_{i\in[k]}d_i,
\max_{i\in[k]}
m_i\left(
\frac{\alpha}{80k},
\frac{\beta}{8k},
1,
\frac{\delta}{64e^2k}
\right)
\right\}
\right).
\]
Applying \Cref{thm:carry-forward-main} with
\(v=\operatorname{VC}(\mathcal C)\) gives
\[
n
=
\widetilde O\left(
\frac{\sqrt{k}}{\varepsilon}
\max\left\{
\frac{k\operatorname{VC}(\mathcal C)}{\alpha},
M_A
\right\}
\right),
\]
and substituting the bound on \(M_A\) proves the claim.
\end{proof}

\subsection*{AI use statement}

In this work, the authors used generative AI through OpenAI's ChatGPT and Codex to discuss and check mathematical arguments and to draft and revise parts of the proofs. It was also used for literature searches, manuscript editing, and the preparation and verification
of the LATEX source.

More importantly, AI helped develop $\textsc{CarryForwardLearner}$ (\Cref{alg:carry-forward}) and prove its privacy guarantee. This assistance came after extensive interaction, starting with an initial incomplete version of the algorithm provided by the authors. 
Beyond that, AI assistance was limited to lower-level tasks. The authors developed the entire road-map of this work, including the abstraction of the generic component-by-component learner and its connection to learning halfspaces and closure under Boolean composition.

AI-assisted suggestions were discussed and revised during manuscript development; they do not replace verification of the mathematical arguments and cited results. The authors are responsible for the correctness, originality, and final content of this work.

\bibliographystyle{plainnat}
\bibliography{references}
\appendix
\crefalias{section}{appendix}
\crefalias{subsection}{appendix}
\crefalias{subsubsection}{appendix}

\section{Quasi-Concavity of the Depth Function}
\label{app:depth-quasiconcavity}

Let \(S\) consist of affine inequalities
\(h_j(y)\le 0\), \(j\in[m]\), over \(y\in\mathbb R^k\), and define
\[
\operatorname{depth}_S(y)
:=
\bigl|\{j\in[m]:h_j(y)\le0\}\bigr|.
\]

We use the following standard convexity fact; see
\citet[Sections~2.3.1--2.3.2]{boyd2004convex}.

\begin{proposition}[Convex feasible sets and convex projections]
\label{prop:depth-projection}
Fix \(i\in[k]\), a prefix \(y_{\le i-1}\), and a set \(T\subseteq[m]\).
Define the feasible completion set
\[
F_T(y_{\le i-1})
:=
\left\{
(z,u)\in\mathbb R\times\mathbb R^{k-i}:
h_j(y_{\le i-1},z,u)\le0
\text{ for every }j\in T
\right\}.
\]
Each constraint defines a halfspace in the remaining variables, and hence
\(F_T(y_{\le i-1})\) is convex. Consequently, its projection onto the
\(z\)-coordinate,
\[
I_T(y_{\le i-1})
:=
\left\{
z\in\mathbb R:
\exists u\in\mathbb R^{k-i}
\text{ with }(z,u)\in F_T(y_{\le i-1})
\right\},
\]
is convex and therefore an interval, possibly empty or unbounded.
\end{proposition}\begin{lemma}[One coordinate depth quasi concave]
\label{lem:realizable-depth-quasiconcave-appendix}
Fix $i\in[k]$, a prefix $y_{\le i-1}$, and a labeled sample $S$ of
size $m$, inducing affine constraints
$h_j(y)\le 0$ for $j\in[m]$. Suppose that
$(y_{\le i-1},S)$ is realizable prefix state; that is, there exists a completion
$(y_i^\star,\ldots,y_k^\star)$ such that
\[
h_j(y_{\le i-1},y_i^\star,\ldots,y_k^\star)
\le 0
\qquad\text{for every }j\in[m].
\]
Define
\[
Q_i^S(z)
:=
\max_{y_{i+1},\ldots,y_k}
\operatorname{depth}_S
\bigl(
y_{\le i-1},z,y_{i+1},\ldots,y_k
\bigr).
\]
Then
\[
Q_i^S:\mathbb R\longrightarrow\{0,1,\ldots,m\}
\]
is quasi-concave. Consequently, its restriction to the finite ordered
domain $\widetilde{\cY}_i(y_{\le i-1})\subseteq\mathbb R$ is quasi-concave with respect
to the induced order. Explicitly, for every
$z_1\le z_2\le z_3$ in $\widetilde{\cY}_i(y_{\le i-1})$,
\[
Q_i^S(z_2)
\ge
\min\bigl\{Q_i^S(z_1),Q_i^S(z_3)\bigr\}.
\]
\end{lemma}

\begin{proof}
For $t\in\{0,1,\ldots,m\}$, define the superlevel set
\[
L_t(y_{\le i-1})
:=
\bigl\{z\in\mathbb R:Q_i^S(z)\ge t\bigr\}.
\]
The claim is immediate for $t=0$. Fix $t\in[m]$. For every
$T\subseteq[m]$ with $|T|=t$, let $I_T(y_{\le i-1})$ denote the
interval from \Cref{prop:depth-projection}. By the
definition of $Q_i^S$,
\begin{equation}
\label{eq:depth-superlevel-union}
L_t(y_{\le i-1})
=
\bigcup_{\substack{T\subseteq[m]\\ |T|=t}}
I_T(y_{\le i-1}).
\end{equation}
Indeed, if $Q_i^S(z)\ge t$, then some completion of $z$ satisfies at
least $t$ constraints. Hence, $z\in I_T(y_{\le i-1})$ for some
$T\subseteq[m]$ of size $t$. Conversely, if
$z\in I_T(y_{\le i-1})$, then some completion of $z$ satisfies all
constraints indexed by $T$, and therefore $Q_i^S(z)\ge t$.

By prefix realizability, there exists a completion
$(y_i^\star,\ldots,y_k^\star)$ satisfying every constraint in $S$.
Consequently,
\[
y_i^\star\in I_T(y_{\le i-1})
\qquad
\text{for every }T\subseteq[m].
\]
Thus, all the intervals in \eqref{eq:depth-superlevel-union} contain
the common point $y_i^\star$. Their union is therefore an interval, so
$L_t(y_{\le i-1})$ is convex.

Since $Q_i^S$ is integer-valued, it suffices to consider integer
superlevel thresholds. Hence, every superlevel set of $Q_i^S$ is
convex, and the superlevel-set characterization of quasi-concavity
\citep[\S3.4.1]{boyd2004convex} implies that $Q_i^S$ is quasi-concave
on $\mathbb R$.

Finally, Since
\[
\widetilde{\cY}_i(y_{\le i-1})\subseteq\mathbb R
\]
is equipped with the order induced by $\mathbb R$, every ordered triple
$z_1\le z_2\le z_3$ in this finite domain is also an ordered triple in
$\mathbb R$. Therefore, quasi-concavity of $Q_i^S$ on $\mathbb R$
implies
\[
Q_i^S(z_2)
\ge
\min\bigl\{Q_i^S(z_1),Q_i^S(z_3)\bigr\}.
\]
Hence the restriction of $Q_i^S$ to
$\widetilde{\cY}_i(y_{\le i-1})$ is quasi-concave with respect to the
induced order.
\end{proof}\section{Proof of the Local Proper Learner Corollary}
\label{app:local-proper-blackbox}

\begin{corollary}
    [Restatement of \Cref{cor:blackbox-local-oracle-statment}]
    Fix \(i\in[k]\)
and a valid prefix \(y_{\le i-1}\), and let \(A_i\) be the component
oracle defined in
\Cref{prop:local-component-oracle-construction-statement}. Throughout
the proof, we assume the structural conditions stated in
\Cref{cor:blackbox-local-oracle-statment}.

\end{corollary}
\begin{proof}
\textbf{Totality and validity.}
The algorithms \textsc{ARelabel} and \(\widehat B_i\) is an empirical leaner from \cref{lem:pac-to-empirical} are defined on
every labeled input. Since \(\widehat B_i\) is proper, its output
satisfies \(y_i\in\mathcal H_i\). Consequently, by the first
structural condition in
\Cref{cor:blackbox-local-oracle-statment},
\[
y_{\le i}=(y_{\le i-1},y_i)\in\mathcal Y_{\le i},
\qquad
\mathcal C(y_{\le i})\neq\emptyset.
\]

\textbf{Privacy.}
The third structural condition in
\Cref{cor:blackbox-local-oracle-statment} implies that \(q_i\) has
matched sensitivity at most one in the sense of
\citet[Definition~24]{alon2020closure}. Indeed, suppose that
\(h,h'\in\mathcal H_i\) agree on the carrier of a common subsample.
Every completion of \(h\) then has a completion of \(h'\) agreeing
with it on that carrier. Hence the two completion scores can differ
only on the additional example, and therefore by at most one. The
reverse direction follows symmetrically.

\cref{lem:pac-to-empirical} learner preserves the privacy parameters of \(B_i\).
Applying Lemma~25 of Alon et al.\ to \textsc{ARelabel} followed by
the private continuation \(\widehat B_i\) therefore gives
\[
(\varepsilon_B+3,\,4e\delta_B)
=
(\varepsilon_A,\delta_A)
\text{-differential privacy}.
\]
This holds on arbitrary labeled inputs and, in particular, on
realizable neighboring inputs as required by
\Cref{ass:component-oracle}.

\textbf{Utility on realizable prefix states.}
Suppose
\[
S=U(S)^{w_{i-1}},
\qquad
w_{i-1}\in\mathcal C(y_{\le i-1}).
\]
By the second structural condition in
\Cref{cor:blackbox-local-oracle-statment}, there exists
\(h_i^\star\in\mathcal H_i\) such that
\(w_{i-1}\in\mathcal C(y_{\le i-1},h_i^\star)\).
Since \(w_{i-1}\) realizes \(S\), it follows from
\eqref{eq:blackbox-score} that \(q_i(S,h_i^\star)=0\). hence $\min_{h_i\in\cH_i}q_i(S,h_i)=0$

Apply Claim~26 of Alon et al.\ with error parameter \(\tau/2\) and
failure probability \(\beta_A/2\).  
Let \(h\in \cH_i\) be the hypothesis
selected by the exponential mechanism, then, except with probability
at most \(\beta_A/2\),
\[
q_i(S,h)\le \min_{h_i\in\cH_i}q_i(S,h_i) +\frac{\tau}{2}|S| \le \frac{\tau}{2}|S| 
\]
where
\[
|S|
=
\widetilde O\!\left(
\frac{d_i+\log(1/\beta_A)}{\tau}
\right).
\]
By definition \(q_i(S,h)\) is the minimum disagreement of hypothesis $g_i\in\cC(y_{\le i-1},h)$, with that there exists
\(\bar g_i\in\mathcal C(y_{\le i-1},h)\) such that disagreement 
$
\operatorname{disc}_{U(S)}(\bar g_i,w_{i-1})
\le\frac{\tau}{2}.
$

Since \(\widetilde S=U(S)^h\) is
\(\mathcal H_i\)-realizable, the empirical learner returns
\(h'\in\mathcal H_i\) such that, except with probability
\(\beta_A/2\),
\[
\operatorname{err}_{\widetilde S}(h')
=
\operatorname{disc}_{U(S)}(h',h)
\le\frac{\tau}{2},
\]
provided that
\(
|S|\ge
9m_{B,i}(\tau/2,\beta_A/2,\varepsilon_B,\delta_B).
\)

By the third structural condition in
\Cref{cor:blackbox-local-oracle-statment}, there exists
\(g_i\in\mathcal C(y_{\le i-1},h')=\mathcal C(y_{\le i})\) such that
\[
\operatorname{disc}_{U(S)}(g_i,\bar g_i)
\le
\operatorname{disc}_{U(S)}(h',h)
\le\frac{\tau}{2}.
\]
The two preceding events each fail with probability at most
\(\beta_A/2\). Hence, by a union bound, with probability at least
\(1-\beta_A\),
\[
\operatorname{disc}_{U(S)}(g_i,w_{i-1})
\le
\operatorname{disc}_{U(S)}(g_i,\bar g_i)
+
\operatorname{disc}_{U(S)}(\bar g_i,w_{i-1})
\le\tau.
\]
Thus \(y_i\) admits a completion \(g_i\in\mathcal C(y_{\le i})\)
satisfying the utility requirement of
\Cref{ass:component-oracle}.
\end{proof}

\section{Halfspaces and Linear Feasibility Reduction}
\label{app:halfspace-linear-feasibility}
For the halfspace application, let
\(\mathcal X=[[\pm X]]\), where \(X\in\mathbb N\), and let
\(L:=|\mathcal X|=2X+1\).
The input domain is the public finite grid
\(\mathcal X^d\subseteq\mathbb R^d\).
The grid is part of the problem specification, not computed from the
private sample. Examples lie in \(\mathcal X^d\), whereas the halfspace
parameters \((a,w)\) lie in \(\mathbb R^{d+1}\).
The oracles evaluate the linear constraints induced by a labeled
sample; the sample itself remains labeled throughout the learner.

\paragraph{Linear-feasibility representation.}
We use the reduction of
\citet[Section~5.1]{kaplan2020halfspaces}, also used by
\citet[Section~5.1]{nissim2026quasiconcave}.
We use the convention \(h_{a,w}(x)=1\) if
\(\langle a,x\rangle\ge w\), and \(-1\) otherwise.
Every labeled example \((x,y)\in\mathcal X^d\times\{-1,1\}\)
induces the constraint
\[
H_{x,y}:=\{(a,w)\in\mathbb R^{d+1}:
y(\langle a,x\rangle-w)\ge0\}.
\]
For \(p\in\mathbb R^{d+1}\), let
\(r_p:=\{H_{x,y}:(x,y)\in\mathcal X^d\times\{-1,1\},\ p\in H_{x,y}\}\).
This is the collection of induced constraints satisfied by \(p\).
The range family \(\{r_p:p\in\mathbb R^{d+1}\}\) has VC dimension
at most \(d+1\) by~\citet[Lemma~3 and Appendix~B]{nissim2026quasiconcave}.
For \(S=((x_j,y_j))_{j=1}^n\), write
\(\operatorname{depth}_S(p):=|\{j\in[n]:p\in H_{x_j,y_j}\}|\).
Thus depth counts satisfied constraints, including repeated occurrences.
These definitions also apply to samples in \(\mathbb R^d\), in
particular to the perturbed samples below.

\begin{lemma}[Realizable halfspace to linear feasibility]
\label{lem:halfspace-to-lp}
Let \(S=((x_j,y_j))_{j=1}^n\in(\mathbb R^d\times\{-1,1\})^n\)
be realizable by \(h^\star_{u,b}\), and let
\(S^{\mathrm{LP}}=(y_j(\langle a,x_j\rangle-w)\ge0)_{j=1}^n\)
be its induced multiset of constraints.
Then \(p^\star=(u,b)\) satisfies every constraint in
\(S^{\mathrm{LP}}\), so \(\operatorname{depth}_S(p^\star)=n\).
Moreover, every submultiset \(B\) of \(m\) occurrences from \(S\)
satisfies \(\operatorname{depth}_B(p^\star)=m\).
\end{lemma}

\begin{proof}
Since \(h^\star_{u,b}\) realizes \(S\), if \(y_j=1\), then
\(\langle u,x_j\rangle-b\ge0\); if \(y_j=-1\), then
\(\langle u,x_j\rangle-b<0\). In both cases,
\(y_j(\langle u,x_j\rangle-b)\ge0\).
Hence \(p^\star\) satisfies all constraints induced by \(S\),
and therefore also all constraints induced by \(B\).
\end{proof}

\begin{corollary}[Restatement of \cref{lemma:realizable-lp-transfer-statement}]
\label{cor:realizable-lp-transfer}
Let \(S=((x_j,y_j))_{j=1}^n\) be a realizable labeled sample,
and let \(B\) be a uniformly random submultiset of \(m\) occurrences
from \(S\). For \(\rho,\beta\in(0,1)\), suppose that
\(m\le n\) and
\(m\ge\frac{c}{\rho}((d+1)\log(2/\rho)+\log(1/\beta))\)
for a sufficiently large universal constant \(c\).
Then, with probability at least \(1-\beta\), every
\(p\in\mathbb R^{d+1}\) satisfying \(\operatorname{depth}_B(p)=m\)
also satisfies \(\operatorname{depth}_S(p)\ge(1-\rho)n\).
\end{corollary}

\begin{proof}
By \Cref{lem:halfspace-to-lp}, a point \(p^\star\) satisfies every
constraint induced by \(S\). Apply
\Cref{thm:realizable-vc-generalization} to the range family
\(\{r_p:p\in\mathbb R^{d+1}\}\), whose VC dimension is at most
\(d+1\). With the stated probability, every range containing all
constraints induced by \(B\) contains all but a \(\rho\)-fraction
of those induced by \(S\).
For repeated examples, distinguish their occurrences as
in~\citet[Theorem~5, Footnote~5]{nissim2026quasiconcave}.
This preserves the VC dimension, realizability, depth, and uniform
sampling, so the same argument applies to multisets.
\end{proof}

\paragraph{Coordinate depth objective.}
Write the halfspace parameters in the order
\(y=(w,a_1,\ldots,a_d)\), with \(k=d+1\).
This only reorders the entries of \(p=(a,w)\).
Recall from \Cref{def:depth-coordinate-function} that
\[
Q_i^S(z):=\max_{y_{i+1},\ldots,y_k\in\mathbb R}
\operatorname{depth}_S(y_{\le i-1},z,y_{i+1},\ldots,y_k).
\]
The maximization is over real parameter completions, not over
example points in \(\mathcal X^d\). The finite domains below restrict
the next coordinate \(z\) supplied to the oracle.

\paragraph{Finite coordinate domains.}
We use the coordinate domains of
\Cref{def:finite-coordinate-domains}, with parameter order
\(y=(w,a_1,\ldots,a_d)\) and
\(\widetilde{\mathcal Y}_1=\{-1,1\}\).
Excluding \(w=0\) rules out the parameter vector \((a,w)=(0,0)\),
which satisfies every induced constraint regardless of the labels.
The proof below shows that some realizing representation remains.

\begin{definition}[finite coordinate domains {\citep[Definition~4.4]{kaplan2020halfspaces}}]
\label{def:kaplan-domain-build}
We define the domains \(\{\widetilde{\mathcal Y}_i\}_{i=1}^{d+1}\) iteratively.
For \(i=1\), let \(\widetilde{\mathcal Y}_1:=\{-1,1\}\).
For \(i=2\), let \(\widetilde{\mathcal Y}_2:=\widetilde{\mathcal Y}'_2/\widetilde{\mathcal Y}''_2\)
where \(\widetilde{\mathcal Y}'_2:=[\![\pm(d\cdot d!)\cdot X^d]\!]\)
and \(\widetilde{\mathcal Y}''_2:=([\![\pm d!\cdot X^d]\!])\setminus\{0\}\).
For \(i>2\) and given \(a_{j-1}^*=s_j/t_j\in\widetilde{\mathcal Y}'_j/\widetilde{\mathcal Y}''_j=\widetilde{\mathcal Y}_j\)
for \(j\in\{2,\ldots,i-1\}\), define
\(\widetilde{\mathcal Y}_i=\widetilde{\mathcal Y}_i(w^*,a_1^*\ldots,a_{i-2}^*):=\widetilde{\mathcal Y}'_i/\widetilde{\mathcal Y}''_i\)
where \(\widetilde{\mathcal Y}'_i:=[\![\pm(d\cdot d!)^{i-1}\cdot X^{d(i-1)}]\!]\)
and \(\widetilde{\mathcal Y}''_i=\widetilde{\mathcal Y}''_i(t_{i-1}):=([\![\pm d!\cdot X^d]\!]\cdot t_{i-1})\setminus\{0\}\).
\end{definition}
\begin{lemma}[Attaining full depth on the finite domain
{\citep[Lemma~4.5, realizable case]{kaplan2020halfspaces}}]
\label{lem:kaplan-domain-full-depth}
Let \(i\in\{2,\ldots,k\}\), and let
\(y_{\le i-1}=(w,a_{\le i-2})\) be generated by the domains
of \Cref{def:kaplan-domain-build}.
Suppose that this prefix admits a completion satisfying
all \(m=|S|\) constraints in \(S_w^{\mathrm{LP}}\).
Then there exists
\(y_i\in\widetilde{\mathcal Y}_i(y_{\le i-1})\)
admitting a completion satisfying all \(m\) constraints.
Equivalently,
\[
Q_i^S(y_i)
=
\max_{z\in\widetilde{\mathcal Y}_i(y_{\le i-1})}Q_i^S(z)
=
m.
\]
\end{lemma}

\begin{proof}
The cited finite-domain guarantee provides a coordinate
\(a_{i-1}\in\widetilde{\mathcal Y}_i(y_{\le i-1})\)
and a full completion \(a\) attaining the maximum
\(\operatorname{cdepth}_{S_w^{\mathrm{LP}}}\)
under the fixed prefix.
This maximum equals \(m\), since the prefix admits a
completion satisfying every constraint.
By \citet[Claim~4.1]{kaplan2020halfspaces},
\[
\operatorname{depth}_{S_w^{\mathrm{LP}}}(a)
\ge
(d+1)\operatorname{cdepth}_{S_w^{\mathrm{LP}}}(a)-dm
=
m.
\]
Thus \(a\) satisfies all \(m\) constraints.
Setting \(y_i=a_{i-1}\) proves the claim.
\end{proof}

\begin{definition}[Independent perturbation]
\label{def:independent-perturbation}
Let \(A\subset\mathbb R\) be a finite nonempty set, and let
\(S=((x_j,y_j))_{j=1}^n\) be a labeled sample.
An independent perturbation of \(S\) from \(A\) replaces each
\((x_j,y_j)\) by \((x_j+z_j,y_j)\), where every coordinate of
every \(z_j\in A^d\) is sampled independently and uniformly from
\(A\).
Thus only the sample points are perturbed; their labels are
retained.
\end{definition}
\begin{lemma}[General position after perturbation
{\citep[Lemma~B.6]{kaplan2020halfspaces}}]
\label{lem:perturbation-general-position}
Let \(S\) be a labeled sample of \(n\) points in \(\mathbb R^d\),
possibly with repeated occurrences, and let
\(\beta_{\mathrm{GP}}\in(0,1)\).
Suppose that \(A\subset\mathbb R\) is finite and {\footnotemark}
\[
|A|\ge \frac{d n^{d+1}}{\beta_{\mathrm{GP}}} \ 
\]

\footnotetext{We use \(n^{d+1}\) in place of the \(n^d\) bound
stated in \citet[Lemma~B.6]{kaplan2020halfspaces}.
The event \(E_k\) in that proof involves \(k+1\) sample
occurrences, rather than \(k\), so the corresponding union
bound must also account for the additional occurrence.
With this correction, the remaining perturbation analysis
is unchanged.}
If \(\widetilde S\) is obtained from \(S\) by
\Cref{def:independent-perturbation}, then, with probability at
least \(1-\beta_{\mathrm{GP}}\), \(U(\widetilde S)\) is in
general position.
\end{lemma}

\begin{lemma}[Realizability under small perturbations]
\label{lem:perturbation-preserves-realizability}
Let \(X\in\mathbb N\), and let
\(S\in([\![\pm X]\!]^d\times\{-1,1\})^n\) be realizable.
Choose $
\Delta:=\left\lceil
\frac{d n^{d+1}}{2\beta_{\mathrm{GP}}}
\right\rceil,
\
\Delta':=\left\lceil 2\Delta X'\sqrt d\right\rceil,
\
A:=\frac{[\![\pm\Delta]\!]}{\Delta'}.
$
If \(\widetilde S\) is obtained by independently perturbing
\(S\) from \(A\), then \(\widetilde S\) remains realizable.
The same holds after a common positive rescaling of its points.
\end{lemma}
\begin{proof}
By \citet[Lemma~B.5]{kaplan2020halfspaces}, there exists
a halfspace realizing \(S\) with margin at least \(1/X'\),
where \(X'\) is the quantity specified in that lemma.
Following the construction in
\citet[Theorem~B.7]{kaplan2020halfspaces}, with the corrected
resolution from \Cref{lem:perturbation-general-position}, set
$
\Delta:=\left\lceil
\frac{d n^{d+1}}{2\beta_{\mathrm{GP}}}
\right\rceil,
\
\Delta':=\left\lceil 2\Delta X'\sqrt d\right\rceil,
\
A:=\frac{[\![\pm\Delta]\!]}{\Delta'}.
$
Every perturbation \(z\in A^d\) satisfies
\(\|z\|_2\le 1/(2X')<1/X'\).
Hence the perturbations preserve realizability, and
\(\widetilde S\) is realizable.
A common positive rescaling of its points also preserves
realizability.
\end{proof}

\begin{corollary}[Perturbation throughout the carry-forward learner]
\label{cor:general-position-throughout}
Let \(\mathcal X^d\subseteq[\![\pm X]\!]^d\), and let
\(S\in(\mathcal X^d\times\{-1,1\})^n\) be realizable.
Obtain \(\widetilde S\) by independently perturbing \(S\)
using \(A\) from \Cref{lem:perturbation-preserves-realizability},
followed by a fixed positive rescaling onto an integer grid.

Then \(\widetilde S\) is realizable, and every active sample
\(S_i\), \(0\le i<k\), generated by \textsc{CarryForwardLearner}
forms a realizable prefix state with \(y_{\le i}\).
Moreover, with probability at least \(1-\beta_{\mathrm{GP}}\),
\(U(\widetilde S)\) and all active samples are in general position.
\end{corollary}

\begin{proof}
By \Cref{lem:perturbation-preserves-realizability},
\(\widetilde S\) is realizable.
By \Cref{lem:perturbation-general-position}, its points are
in general position with probability at least
\(1-\beta_{\mathrm{GP}}\); the rescaling preserves this property.

On this event, every carrier selected without replacement
from \(\widetilde S\) inherits general position.
Relabeling leaves its points unchanged:
\(U(S_0)=U(R_0)\) and \(U(S_i)=U(T_i)\).
Finally, \Cref{lem:iterated-realizability} gives realizability
of every prefix state in every execution.
The same general-position event therefore covers all rounds.
\end{proof}
\paragraph{Finite coordinate domains after perturbation.}
\label{par:finite-domains-after-perturbation}
We use the domains of \Cref{def:kaplan-domain-build} with
\(y=(w,a_1,\ldots,a_d)\).
By \Cref{cor:general-position-throughout}, every active prefix
state remains realizable after perturbation and rescaling, and
its sample is in general position on the stated high-probability
event.
Thus, for every such prefix and \(i\ge2\),
\Cref{lem:kaplan-domain-full-depth} guarantees an element of
\(\widetilde{\mathcal Y}_i(y_{\le i-1})\) attaining \(Q_i^S=m\).
Following \Cref{lem:perturbation-preserves-realizability},
rescaling by \(\Delta'\) places the perturbed points in
\([\![\pm\widetilde X]\!]^d\), where
\(\widetilde X:=\Delta'(X+\Delta)\).
The corrected resolution increases this grid parameter
by at most a factor \(O(n^2)\) compared with the construction
in \citet[Theorem~B.7]{kaplan2020halfspaces}.

Recall that \(L=|\mathcal X|\) is the original coordinate grid size.
Let \(N\) be the maximum finite coordinate-domain size after
perturbation and rescaling, whose grid has coordinate bound
\(\widetilde X=\Delta'(X+\Delta)\),
with $\Delta=\left\lceil
\frac{d n^{d+1}}{2\beta_{\mathrm{GP}}}
\right\rceil,
\
\Delta'=\left\lceil 2\Delta X'\sqrt d\right\rceil$ by \citet[Lemma ~B.5]{kaplan2020halfspaces} $X'=2d^2\cdot d!^{d^3}\cdot X^{d^6}$. using $\Delta' \le 4\Delta X'\sqrt{d}, \Delta+X\le 2\Delta X$ gives $\tilde{X}\le 16 d^{9/2}(d!)^{d^3}X^{d^6+1}\frac{n^{2d+2}}{\beta_{GP}^2}$.
For every generated prefix and \(j\in[d]\), the construction in \Cref{def:finite-coordinate-domains} gives
\(
|\widetilde{\mathcal Y}_{j+1}(y_{\le j})|
\le
(2(dd!)^j\widetilde X^{dj}+1)(2d!\widetilde X^d).
\)
Taking the maximum over all generated prefixes and coordinates
gives \(N\le6d^d(d!)^{d+1}\widetilde X^{d^2+d}\). $L = |\cX| = (2X +1)$ we use $L \ge X$ therefore $N\le
(2d)^{O(d^6)}
L^{O(d^8)}
n^{O(d^3)}
(1/\beta_{\mathrm{GP}})^{O(d^2)}.
\)
Consequently,
\(
\log^\ast N
\le
\max\{\log^\ast L,\log^\ast d,\log^\ast n,
\log^\ast(1/\beta_{\mathrm{GP}})\}+O(1).$

Since \(2^{\log^\ast u}=O(1+\log u)\), the additional terms
introduce only logarithmic factors in \(n,d\) and
\(1/\beta_{\mathrm{GP}}\).
Thus, replacing the dependence on \(N\) by dependence on \(L\),
both for \(\log^\ast N\) and \(2^{\log^\ast N}\), affects only
factors suppressed by \(\widetilde O\) in the final learning bounds.

\begin{lemma}[Restatement of
\cref{lem:normalized-finite-domains-statement}]
\label{lem:normalized-finite-domains}
Let \(i\in[k]\), and let \((y_{\le i-1},S)\) be a realizable
prefix state with \(m=|S|\), where the prefix is generated by
the domains of \Cref{def:finite-coordinate-domains}.
Then \(Q_i^S\) is sensitivity-one and quasi-concave on
\(\widetilde{\mathcal Y}_i(y_{\le i-1})\).
Moreover, there exists
\(y_i\in\widetilde{\mathcal Y}_i(y_{\le i-1})\)
such that \(y_{\le i}\) admits a completion satisfying all
\(m\) constraints induced by \(S\).
Equivalently,
\[
Q_i^S(y_i)
=
\max_{z\in\widetilde{\mathcal Y}_i(y_{\le i-1})}Q_i^S(z)
=
m.
\]
\end{lemma}

\begin{proof}
Sensitivity follows from
\Cref{def:depth-coordinate-function}.
By realizability and \Cref{lem:halfspace-to-lp}, the fixed
prefix admits a completion satisfying all induced constraints.
Quasi-concavity therefore follows from
\Cref{lem:realizable-depth-quasiconcave-appendix}.

It remains to show that the finite domain contains a coordinate
admitting a completion of depth \(m\).

For \(i=1\), we choose the threshold \(w\).
A positive rescaling of any realizing halfspace lets us assume
that \(w\in\{-1,0,1\}\).
If \(w\in\{-1,1\}\), then \(w\in\widetilde{\mathcal Y}_1\)
already, and the claim follows from
\Cref{lem:halfspace-to-lp}.

Suppose that \(w=0\), and let \(c_{a,0}\) realize \(S\).
If \(S\) contains no negatively labeled point, then
\(c_{0,-1}\) realizes \(S\).
Otherwise, define
\[
\gamma
:=
\min_{(x,-1)\in S}\bigl(-\langle a,x\rangle\bigr)>0,
\]
and choose \(\lambda>1/\gamma\).
Every positively labeled point satisfies
\(\langle \lambda a,x\rangle\ge0>-1\), while every negatively labeled
point satisfies
\(\langle \lambda a,x\rangle\le-\lambda \gamma<-1\).
Thus \(c_{\lambda a,-1}\) realizes \(S\).
In either case, some \(w\in\{-1,1\}\) admits a realizing
completion and hence a completion satisfying all \(m\)
constraints by \Cref{lem:halfspace-to-lp}.

For \(i>1\), the threshold \(w\) is fixed.
Realizability and \Cref{lem:halfspace-to-lp} imply that the
remaining coefficients admit a completion satisfying every
by
\Cref{lem:kaplan-domain-full-depth}, the finite domain
\(\widetilde{\mathcal Y}_i(w,a_{\le i-2})\)
contains a coordinate \(a_{i-1}\) admitting a completion
with the same maximum value.
\end{proof}

\paragraph{Returning to halfspaces.}

\begin{corollary}[Depth-to-disagreement transfer]
\label{cor:lp-to-halfspace-appandix}
Let \((y_{\le i-1},S)\) be a realizable prefix state witnessed
by \(h_{i-1}\), and suppose that \(U(S)\) is in general position.
Let \(m=|S|\).
If \(y_{\le i}\) admits a completion
\(g_i=h_{a,w}\in\mathcal C(y_{\le i})\), with
\(w\in\{-1,1\}\), satisfying at least \(m-E\) induced constraints,
then \(g_i\) makes at most \(E+d\) classification errors on \(S\),
and
\[
\operatorname{disc}_{U(S)}(g_i,h_{i-1})
\le\frac{E+d}{m}.
\]
\end{corollary}
\begin{proof}
A classification error that does not violate its induced
constraint can occur only at a negatively labeled point with
\(\langle a,x\rangle=w\).
If \(a\ne0\), general position bounds the number of such points
by \(d\); if \(a=0\), there are none because \(w\ne0\).
Thus \(g_i\) makes at most \(E+d\) errors on \(S\).
Since \(h_{i-1}\) realizes \(S\), these errors are precisely
the disagreements between \(g_i\) and \(h_{i-1}\) on \(U(S)\).
Hence
$
\operatorname{disc}_{U(S)}(g_i,h_{i-1})
\le \frac{E+d}{m}.
$
\end{proof}

\paragraph{Utility after perturbation.}
Let \(\widetilde{\mathcal D}\) denote the distribution obtained
by drawing \((x,y)\sim\mathcal D\), independently perturbing
\(x\) from \(A\), and retaining \(y\), before rescaling.
A halfspace returned on the rescaled grid is interpreted in
the original coordinates by composing it with the fixed
rescaling.
Its error on the perturbed distribution then equals the error
of the returned halfspace on the rescaled distribution.

\begin{corollary}[Utility on the original distribution]
\label{cor:halfspace-perturbation-utility}
Let \(\mathcal D\) be a realizable distribution on the original sample
\(\mathcal X^d\times\{-1,1\}\).
Run \textsc{CarryForwardLearner} with the perturbation and
rescaling of \Cref{cor:general-position-throughout} and
component oracles whose utility guarantees hold on realizable
prefix states in general position.
Set \(\beta_{\mathrm{GP}}=\beta/2\),
\(\beta_A,\beta_{\mathrm{LB}},\beta_{\mathrm{VC}}\le\beta/(8k)\),
and \(\tau=\alpha/(80k)\), and choose \(m\) to satisfy the
active-sample requirements of \Cref{thm:iterated-utility}.
Then, with probability at least \(1-\beta\), the returned
halfspace \(h_k\), interpreted in the original coordinates,
satisfies
\[
\operatorname{err}_{\mathcal D}(h_k)\le\alpha.
\]
\end{corollary}

\begin{proof}
Apply \Cref{lem:perturbation-preserves-realizability} to the
entire finite grid, labeled by a halfspace realizing
\(\mathcal D\).
The resulting realizing halfspace agrees with the retained
labels under every allowed perturbation.
Hence \(\widetilde{\mathcal D}\) is realizable.

Independent perturbations preserve i.i.d.\ sampling.
Each carrier is marginally i.i.d.\ from the perturbed marginal,
because its indices are selected independently of the record
values.
Thus the all-pairs generalization argument in
\Cref{thm:iterated-utility} applies, including after rescaling
and adaptive relabeling.
This sampling argument is applied without conditioning on
general position.

By \Cref{cor:general-position-throughout}, all active samples
are in general position except with probability
\(\beta_{\mathrm{GP}}\).
Combining this event with the oracle and generalization
events in the proof of \Cref{thm:iterated-utility}, the total
failure probability is at most
\[
\beta_{\mathrm{GP}}
+k\beta_A+k\beta_{\mathrm{LB}}+(k+1)\beta_{\mathrm{VC}}
\le\beta.
\]
On simultaneous success,
\[
\operatorname{err}_{\widetilde{\mathcal D}}(h_k)
\le20(k+1)\tau
\le\frac{\alpha}{2}.
\]

It remains to transfer this guarantee to \(\mathcal D\).
For any \((x,y)\) misclassified by \(h_k\), and any \(z\in A^d\),
at least one of \(x+z\) and \(x-z\) is also misclassified
with respect to the retained label \(y\).
Indeed, if both were correctly classified, their midpoint
\(x\) would lie on the same correctly labeled side of the
separating hyperplane.
Since \(A\) is symmetric, \(z\) and \(-z\) have the same
probability.
Thus every misclassified example remains misclassified after
perturbation with probability at least \(1/2\).
Averaging over \(\mathcal D\), as in
\citet[the proof of Theorem~B.7]{kaplan2020halfspaces}, gives
\[
\operatorname{err}_{\mathcal D}(h_k)
\le
2\operatorname{err}_{\widetilde{\mathcal D}}(h_k)
\le\alpha.
\]
\end{proof}
\begin{proof}[Proof of \Cref{thm:halfspace-framework}]
\label{proof:under-genereal-position}
Apply the perturbation with \(\beta_{\mathrm{GP}}=\beta/2\),
and run \textsc{CarryForwardLearner} with parameters
\((\alpha/2,\beta/2,\varepsilon,\delta/2)\).
Thus \(\tau=\alpha/(80k)\), \(\beta_A=\beta/(8k)\),
\(\varepsilon_A=4\), and \(\delta_A=\delta/(32ek)\).

The oracle guarantees are required only on prefixes generated by
\Cref{def:finite-coordinate-domains}.
Indeed, each oracle is defined on every labeled input and returns
a coordinate from its prescribed finite domain.
Starting from the empty prefix, all prefixes therefore remain
within the construction, regardless of utility success.

Validity follows from the coordinate representation:
each nonempty prefix has threshold \(w\in\{-1,1\}\),
and setting the remaining coefficients to zero gives a halfspace
completion. Thus every restricted class is nonempty, and the
final output is proper.

The induction in \Cref{lem:iterated-realizability} consequently
still applies. The preliminary call produces
\(S_0=U(R_0)^{h_0}\) with \(h_0\in\mathcal C\).
For \(i<k\), LabelBoost selects
\(h_i\in\mathcal C(y_{\le i})\) and sets
\(S_i=U(T_i)^{h_i}\), making the next state realizable.

For each fixed prefix within the construction, oracle privacy
holds on every pair of neighboring realizable samples, including
those compared in the LabelBoost continuation analysis.
Since all further oracle outputs remain within the construction,
the privacy proof of \Cref{thm:carry-forward-main} applies and
gives \((\varepsilon,\delta/2)\)-differential privacy.
Using the same perturbation noise on corresponding records
preserves adjacency, so the complete learner is also
\((\varepsilon,\delta)\)-differentially private.

By \Cref{cor:general-position-throughout}, all active samples
are in general position except with probability
\(\beta_{\mathrm{GP}}\).
Together with iterated realizability, this allows the oracle
utility guarantees to be used at every invocation.
\Cref{cor:halfspace-perturbation-utility} then gives error at most
\(\alpha\) on the original distribution with probability
at least \(1-\beta\).

Finally, substituting the chosen parameters and
\(\operatorname{VC}(\mathcal C)=O(d)\) into the sample requirement
of \Cref{thm:carry-forward-main} gives the stated bound.
\end{proof}
\end{document}